\documentclass{article} %
\usepackage{iclr2027_conference,times}

\usepackage{amsmath,amsfonts,bm}

\def\eqref#1{equation~\ref{#1}}

\def\1{\bm{1}}

\DeclareMathAlphabet{\mathsfit}{\encodingdefault}{\sfdefault}{m}{sl}
\SetMathAlphabet{\mathsfit}{bold}{\encodingdefault}{\sfdefault}{bx}{n}

\usepackage{hyperref}
\usepackage{url}

\title{Ambient Discrete Diffusion: \\ Using the Wrong Data at the Right Time for Data Efficient Learning}

\author{ \\
\textbf{Julian Kleutgens}$^{1,2,3}$,
\textbf{Mauricio Tec}$^{2}$,
\textbf{Claudio Battiloro}$^{2}$,
\textbf{Francesca Dominici}$^{2}$, \\
\textbf{Giannis Daras}$^{3}$ \\
$^{1}$ ETH Zürich, $^{2}$ Harvard University, $^{3}$ MIT
}

\usepackage{graphicx}
\usepackage{subcaption}
\usepackage{float}
\usepackage{amsmath,amssymb,amsfonts,bm,mathtools,amsthm}
\usepackage{array}
\newcolumntype{M}[1]{>{\centering\arraybackslash}m{#1}}
\usepackage{multirow}
\usepackage[T1]{fontenc}

\usepackage{times}
\usepackage{amsmath}
\usepackage{tikz}
\usetikzlibrary{arrows.meta,decorations.pathreplacing,calc,fadings}
 
\definecolor{TgtC}{RGB}{ 27,120, 90}
\definecolor{SrcC}{RGB}{200, 70, 45}
\definecolor{ColC}{RGB}{140, 60,155}
\definecolor{ArrC}{RGB}{ 40, 85,175}
\definecolor{MaskC}{RGB}{130,130,138}

\usepackage{algorithm}
\usepackage{algpseudocode}
\usepackage[utf8]{inputenc}

\newtheorem{theorem}{Theorem}
\newtheorem{lemma}[theorem]{Lemma}

\newtheorem{assumption}[theorem]{Assumption}

\usepackage{thmtools}
\usepackage{booktabs}       %

\newif\ifdraftnotes \draftnotesfalse

\usepackage{fvextra}   %
\DefineVerbatimEnvironment{sample}{Verbatim}%
{breaklines=true, breakanywhere=true, breaksymbolleft={},fontsize=\footnotesize}

\usepackage{soul}
\sethlcolor{yellow}

\iclrfinalcopy %
\begin{document}

\maketitle

\begin{abstract}
We introduce RefineMix, a framework for training discrete diffusion models under severe data scarcity, a common constraint in scientific applications. RefineMix uses out-of-distribution data at selected diffusion times to improve generalization without biasing the sampling distribution. Although this strategy has been explored in continuous diffusion, discrete diffusion presents a distinct challenge: unlike Gaussian noise, masking preserves domain information in surviving tokens, limiting the use of related data at high noise levels. At low noise levels, however, the domains’ effectively disjoint supports become an advantage, allowing the model to learn from both in-domain and out-of-distribution data without biasing the sampler. We formalize these intuitions and provide a theoretical analysis for the proposed method. Experimentally, across five domain-shift settings, RefineMix matches or outperforms in-domain finetuning and data mixing. For protein sequence generation, finetuning with just 197 in-domain examples nearly doubles the fraction of generated proteins that are simultaneously novel, foldable, and in-family compared to standard finetuning.
\end{abstract}

\section{Introduction}
Discrete diffusion modeling has recently gained significant research traction due to its potential for faster sampling compared to the autoregressive paradigm~\citep{song2025seed,labs2025mercury}. 
It has already been successfully applied to language modeling~\citep{austin2023structureddenoisingdiffusionmodels,lou2024discrete,sahoo2024simpleeffectivemaskeddiffusion,shi2025simplifiedgeneralizedmaskeddiffusion,ou2025your,sahoo2025diffusionduality,sahoo2025esoteric}, molecules~\citep{lee2025genmol} and graphs~\citep{liu2023generative} among other modalities. Beyond the potential for fast sampling, recent evidence~\citep{ni2025diffusion, gao2025makes, prabhudesai2026diffusion} suggests that discrete diffusion models might be more data-efficient than their autoregressive counterparts. However, the performance and the generalization capabilities of the final model still depend on the amount of available training data points~\citep{kaplan2020scalinglawsneurallanguage,hoffmann2022trainingcomputeoptimallargelanguage}. For certain application domains, data is scarce, which limits our ability to train effective models. 
The problem of developing techniques for training discrete diffusion models in data-challenging settings remains underexplored. 

For continuous diffusion and autoregressive models, the standard solution in data-challenging settings is to augment the data pool by borrowing samples from \emph{related} domains with wider sample availability. A common practice is to insert related data at a fixed, uniform rate during training, a technique known as \textit{data-mixing}~\citep{albalak2023efficientonlinedatamixing}. Data-mixing improves upon training on target data alone, but also biases the learned distribution. An alternative solution to data-mixing is to pre-train on all the available data and then finetune to the target domain. This finetuning approach is prone to overfitting issues that must be controlled by carefully tuning the learning rate, the number of parameters to be finetuned, and the duration of the finetuning. 
For continuous diffusion, more rigorous techniques for training generative models with limited and out-of-distribution samples have been developed~\citep{daras2025ambientdiffusionomnitraining, daras2025how, ambient_diffusion, ambient_dps, hosseintabar2025diffem, dataloops, sfbdomni, kawar2023gsure, aali2023solving, modi2025generative, park2025measurement, matsuzaki2026score, sfbd, rozet2024learning, em_diffusion, zhang2025restoration, tewari2023diffusion, wei2026ambient}. One key idea, proposed in Ambient Diffusion Omni~\citep{daras2025ambientdiffusionomnitraining}, is to use out-of-distribution samples only for certain diffusion times, leveraging the fact that noise contracts distributional differences and hence even out-of-distribution samples can be useful for learning in the high-noise regime, while images that agree with the target locally can also be used at low noise. These techniques have not been explored in the context of discrete diffusion.

\begin{figure}
    \centering
    \includegraphics[width=0.8\linewidth]{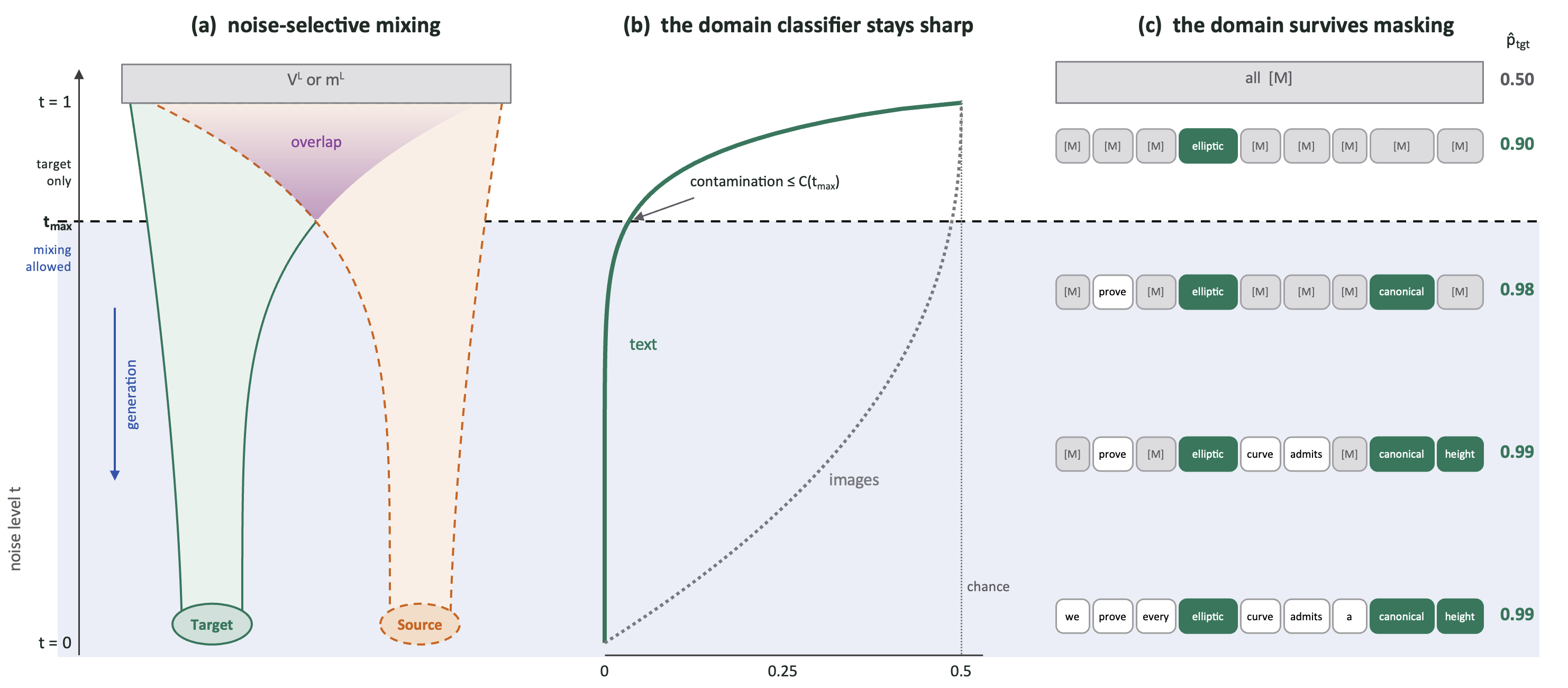}
    \caption{\textbf{Noise-selective mixing.}
    \textbf{(a)} Noising widens the target (green) and source (orange) distributions until both reach the prior $\boldsymbol\pi$ at $t = 1$. They overlap (purple) only at high noise, so RefineMix admits source data only below the cutoff $t_{\max}$ (blue). \textbf{(b)} How well a classifier tells the domains apart, against the noise level: for text (solid), it stays sharp until almost every token is masked and for images under Gaussian noise (dotted), it goes blind early (same-data comparison in Appendix~\ref{app:beyond}). \textbf{(c)} Why text stays separable: masked sequences at decreasing noise (top to bottom), with the classifier's probability of the target. A single target-specific token (green) settles the domain, whereas shared tokens (white) add little.}
    \label{fig:main-figure}
\end{figure}

In this paper, we rigorously study the problem of training \emph{discrete} diffusion models in data-challenging regimes.  
Our central insight is that, contrary to continuous diffusion, in discrete diffusion for textual data, the noising mechanism is not an effective way to erase distribution shifts: the noised domains remain \emph{separable}, and this separability decides where related data can be used. Two observations motivate our strategy. First, we observe that even under extreme token masking, different domains remain separable due to the existence of domain-specific tokens. 
Second, at low noise, different domains do not have any shared data distributional support, and hence, training on data from both domains does not hurt the denoising capabilities on the target domain. On the contrary, if the domains share some local structure, data-mixing at low noise can \textit{improve learning} on the target domain without introducing sampling biases.

These insights lead to \textbf{RefineMix}, a simple framework for training discrete diffusion models in data-scarce domains by leveraging out-of-distribution data without biasing the sampling distribution. RefineMix uses data only from the target domain in the high-noise regime and data from both domains in the low-noise regime. Notably, unlike in continuous diffusion, where Gaussian noise makes the domains indistinguishable at high noise so that imperfect data can also be used there~\citep{daras2025ambientdiffusionomnitraining}, in discrete diffusion the domains remain separable until almost every token is masked, so the related data can only be used at low noise. We justify this design in two steps. First, we prove that sampling stays on target as long as a classifier can still tell the domains apart at the cutoff. Second, we characterize the best cutoff, which trades the structure that the domains share, and that the related data can therefore teach, against the bias it introduces. 

Backing the theory with empirical results, we show that RefineMix matches or outperforms fine-tuning, data mixing, and conditional training on five domain pairs. We further use RefineMix to adapt pretrained diffusion language models of up to $7$B parameters to scarce domains, where it improves on carefully tuned fine-tuning in most settings and matches it in the rest. Finally, we show that the method is not limited to text and we move our interest to modeling protein sequences. We use RefineMix to finetune a diffusion protein model to a family of $197$ proteins and we nearly double the fraction of generated sequences that are new, fold, and belong to the family.

\section{Background}\label{sec:background}

A diffusion model corrupts clean data with a forward process and learns to reverse it. In discrete diffusion~\citep{austin2023structureddenoisingdiffusionmodels}, each token is a one-hot vector $\mathbf z \in \mathcal V$ over a vocabulary of size $N$. We write $\mathrm{Cat}(\cdot\,;\mathbf u)$ for the categorical distribution with probability vector $\mathbf u$ in the simplex $\Delta^{N-1}$, whose vertices form $\mathcal V$. The forward process mixes the clean token with a fixed prior $\boldsymbol\pi \in \Delta^{N-1}$:
\begin{equation}\label{equ_bg:forward}
p(\mathbf z_t \mid \mathbf z_0) = \mathrm{Cat}\big(\mathbf z_t;\, \alpha_t\, \mathbf z_0 + (1-\alpha_t)\, \boldsymbol\pi\big),
\end{equation}
where $\alpha_t$ decreases from $\alpha_0 = 1$ to $\alpha_1 = 0$, so that $\mathbf z_1 \sim \mathrm{Cat}(\boldsymbol\pi)$ carries no information about $\mathbf z_0$ (Appendix~\ref{app:kernels} gives the underlying Markov chain). The \emph{masked} kernel takes $\boldsymbol\pi = \mathbf m$, the one-hot vector of a \texttt{[MASK]} token~\citep{sahoo2024simpleeffectivemaskeddiffusion,shi2025simplifiedgeneralizedmaskeddiffusion}, and the \emph{uniform} kernel takes $\boldsymbol\pi = \mathbf 1/N$~\citep{austin2023structureddenoisingdiffusionmodels,lou2024discrete,schiff2025simpleguidance,sahoo2025diffusionduality}. A network $\mathbf x_\theta(\mathbf z_t) \in \Delta^{N-1}$ predicts the clean token, and the reverse step plugs it into the posterior of the forward process, $p_\theta(\mathbf z_s \mid \mathbf z_t) \coloneqq p(\mathbf z_s \mid \mathbf z_t, \mathbf z_0 = \mathbf x_\theta(\mathbf z_t))$ for $s < t$. Training minimizes the negative ELBO (NELBO), a sum of per-step KL divergences over a grid of $T$ steps $s(i) < t(i)$,
\begin{equation}\label{equ_bg:loss}
\mathcal{L}_{\text{diffusion}} = \sum_{i=1}^{T} \mathbb{E}\Big[D_{\mathrm{KL}}\big(p(\mathbf{z}_{s(i)}\mid \mathbf{z}_{t(i)},\mathbf{z}_0)\,\big\|\, p_{\theta}(\mathbf{z}_{s(i)}\mid \mathbf{z}_{t(i)})\big)\Big].
\end{equation}

\noindent\textbf{Sequences.} So far we have used the notation of tokens, but most categorical domains consist of sequences of tokens. A clean sequence $x_0 = [\mathbf z^1_0, \dots, \mathbf z^L_0]$ is noised token by token, so $p(x_t \mid x_0) = \prod_{\ell} p(\mathbf z^\ell_t \mid \mathbf z^\ell_0)$, and throughout, $x$ denotes sequences and $\mathbf z$ tokens. The network sees the whole noised sequence and predicts every clean token, and the loss \eqref{equ_bg:loss} is summed over positions.

\noindent\textbf{Problem definition.} We study the setting where we are given a few samples from a target domain and abundant samples from a source domain. Formally, we are given a small set of $N_T$ samples from the desired target distribution $p$ and $N_S \gg N_T$ samples from a distinct source distribution $q$. Throughout, we use $\hat p$ and $\hat q$ to denote the empirical distributions. We use $p_t(x_t) = \sum_{x_0} p(x_t \mid X_0=x_0)\, p(x_0)$ to denote the noisy distribution of the target domain at diffusion time $t$ and similarly $q_t(x_t)$ for the source domain distribution at time $t$. To avoid notation clutter, we use $p(x_t)$ and $p_t(x_t)$ interchangeably.

\section{Related Work}\label{sec:related}
Our work focuses on training a model on a scarce domain while borrowing a related data corpus. Discrete diffusion models are commonly adapted to a new task or domain by post-training~\citep{zhao2025d1,wang2025finetuningdiscretediffusionmodels}, noise-level-dependent parameters~\citep{wang2026nara}, or guidance at sampling time~\citep{nisonoff2025unlocking,schiff2025simpleguidance}, or they are converted from a pretrained autoregressive model~\citep{gong2025scaling,ye2025dream}. Closest to our setting is guided transfer learning~\citep{kleutgens2026guidedtransferlearningdiscrete}, which freezes a denoiser trained on the source and trains a small ratio network on the target to guide it, as transfer-guided diffusion does for images~\citep{ouyang2024transfer}. All these methods, however, act only after the training data is fixed, so the related corpus enters only through the pretrained weights and is never trained on alongside the scarce target data. Training on both corpora at once is the idea of data mixing, which pools the scarce domain with related corpora in proportions chosen by a proxy model~\citep{xie2023doremi}, adapted online~\citep{albalak2023efficientonlinedatamixing}, or fitted from the loss of small mixtures~\citep{liu2024regmix}. Domain adaptation theory explains why such pooling helps and when it hurts: the added source samples lower the variance of the estimate, while the distance between the domains sets the bias~\citep{bendavid2010theory}. To keep the domains apart at generation time, the pooled model is often conditioned on a domain label~\citep{ho2022classifier}. When the source is the pretraining data, this pooling is replay, a standard defense against forgetting~\citep{rolnick2019experience,ouyang2022training}. Applied to diffusion, both use one proportion at every noise level, and hence pull samples toward the source (Theorem~\ref{thm:safe}). Ambient diffusion~\citep{consistent_diffusion_meets_tweedie,daras2025ambientdiffusionomnitraining} limits this bias by using imperfect data only where it looks like the target, mostly at high noise, where Gaussian noise erases the difference. Ambient Diffusion Omni~\citep{daras2025ambientdiffusionomnitraining} also uses out-of-distribution images at low noise, up to the level at which a classifier can no longer tell their crops from target crops. Masking keeps the domains separable, so we use the source only at low noise and with the opposite criterion: up to the level at which a classifier can still tell the domains apart, so that the denoiser knows which domain it is completing. RefineMix therefore combines the two lines: it trains on a related corpus as data mixing does, but only at the noise levels where that corpus cannot pull sampling off target.

\section{RefineMix Method and Theoretical Justification}\label{sec:method}

In this Section, we introduce our RefineMix algorithm, we develop intuition around it, and finally we provide a theoretical justification for the proposed method under certain distributional assumptions. %

\noindent\textbf{Setup and baseline algorithms.} As discussed in the Introduction, training on the target data alone leads to poor performance due to overfitting. Merging both corpora and training on all the data avoids overfitting but biases the learned distribution, as the trained discrete diffusion can produce source text as readily as target text. The other standard option, pre-training on everything and then finetuning on the target-only data, achieves a middle ground between overfitting and biased sampling that is in practice controlled by carefully tuning several hyperparameters such as the learning rate and the duration of the finetuning. 

\subsection{The RefineMix algorithm}
RefineMix improves over these baselines by deciding, at each noise level, whether to mix the training data. Recall that a standard training step of discrete diffusion picks a clean sequence, noises it to a random level $t$, and trains the network to denoise it with the loss \eqref{equ_bg:loss}. RefineMix deviates from this implementation by changing the pool of sentences that we can sample from based on $t$. Below a cutoff $t_{\max}$, the sequence comes from the source with probability $\lambda$ and from the target otherwise, and above it, only from the target. Algorithm~\ref{alg:refinemix} defines the whole method. We remark that the schedule, the loss, the network, and the sampler remain as in regular discrete diffusion training.

\begin{algorithm}[t]
\caption{RefineMix training step}\label{alg:refinemix}
\begin{algorithmic}[1]
\Require target samples $\hat p$, source samples $\hat q$, cutoff $t_{\max}$, source rate $\lambda$, denoiser $\mathbf x_\theta$
\State draw $t \sim \mathcal U(0,1)$
\State draw $x_0 \sim \lambda\, \hat q + (1-\lambda)\, \hat p$ if $t \le t_{\max}$, else $x_0 \sim \hat p$
\State noise $x_0$ to $x_t \sim p(x_t \mid x_0)$
\State take a gradient step on the loss \eqref{equ_bg:loss} at level $t$
\end{algorithmic}
\end{algorithm}

\noindent\textbf{Building intuition I: why mixing at high-noise levels hurts.}
To build some intuition for the algorithm, it is useful to think of how the model behaves at inference time. During generation, the time flows backwards: from pure noise down to clean text. If two domains have different vocabularies (e.g. domain specific tokens), the appearance of a single such token during the early inference steps fully controls the domain at which the rest of the generation will arrive. Hence, RefineMix trains at high noise using only the target domain data to ensure that only the target domain will be generated. 

\noindent\textbf{Building intuition II: why mixing at low-noise levels helps.}
After a certain diffusion time, which we call $t_{\max}$, with high probability a target-specific token has appeared and hence there is no risk of drifting from the target domain. Hence, for the remaining steps, using a model trained on data from both domains is harmless. If anything, if the two domains share any local or latent structure, weight sharing enables synergistic learning in such times.
Figure~\ref{fig:main-figure}b,c visualizes this analysis. 

\noindent\textbf{Building intuition III: discrete vs continuous diffusion.}
It is useful to contrast what happens in discrete diffusion over text with continuous diffusion over images, where ambient diffusion uses noisy data only at high noise, with a modified loss~\citep{consistent_diffusion_meets_tweedie}. If the image distributions differ in their high frequencies, due to the spectral power law of natural images, Gaussian noise will quickly erase their differences so at high noise the two distributions become indistinguishable~\citep{daras2025ambientdiffusionomnitraining}. Masking and scrambling do not erase the domain of a text because a single domain-specific token can give away the origin of the sentence. The related data also differs in kind: low-quality images are wrong in their fine detail, whereas a related corpus is off topic but correct in its fine structure, such as grammar and spelling. Hence it helps at low noise, which is also where the domains are separable. The rest of this section makes the intuition precise. Theorem~\ref{thm:safe} determines what mixing costs in sampling, and Section~\ref{sec:method-tradeoff} shows what it costs in likelihood and what it can gain.

\subsection{Theoretical Justification}\label{sec:method-setup}
We now make this intuition precise by separating the effect of mixing into what it can cost and gain.

\noindent\textbf{Setup.} To analyze Algorithm~\ref{alg:refinemix}, we rewrite it in the notation of Section~\ref{sec:background}. Below the cutoff, the algorithm draws a source sequence with probability $\lambda$ and a target sequence otherwise, so it trains on the mixture $m(x_0) = \lambda\, q(x_0) + (1-\lambda)\, p(x_0)$. Because both domains share the forward process, the noised mixture is $m(x_t) = \lambda\, q(x_t) + (1-\lambda)\, p(x_t)$, and Bayes' rule gives the mixed denoiser $m(x_0 \mid x_t) = p(x_t \mid x_0)\, m(x_0) / m(x_t)$. The denoiser of RefineMix is therefore
\begin{equation}\label{eq:refinemix-denoiser}
p_{\text{RefineMix}}(x_0 \mid x_t) =
\begin{cases}
m(x_0 \mid x_t), & t \le t_{\max},\\[2pt]
p(x_0 \mid x_t), & t > t_{\max}.
\end{cases}
\end{equation}
Finally, we assume that the two domains are distinct, meaning that no sequence can come from both: $p(x_0)\,q(x_0) = 0$ for every $x_0$. Since $p$ and $q$ are population distributions, this describes an ideal model; the trained network approximates it.

\subsubsection{What Mixing Can Cost}\label{sec:method-purity}
We now determine how often RefineMix generates a sample from the wrong domain. The result rests on one observation: to complete a noised sequence, the mixed denoiser must first infer which domain the sequence came from.

\noindent\textbf{Classifier.} Since the mixed denoiser must infer the domain, we define a classifier $c(x_t)$: given a noised sequence $x_t$ from either domain, it predicts which domain the sequence came from. We label the source $y=1$ and the target $y=0$, so that under the mixture $m$ the label has prior $\mathbb{P}(y=1)=\lambda$. By Bayes' rule, the probability that $x_t$ came from the source is:
\begin{equation}\label{eq:resp}
c(x_t) \coloneqq \mathbb{P}(y = 1 \mid x_t)
= \frac{\mathbb{P}(x_t \mid y = 1)\, \mathbb{P}(y = 1)}{\mathbb{P}(x_t)}
= \frac{\lambda\, q(x_t)}{m(x_t)}.
\end{equation}

\noindent\textbf{Fact 1: the mixed denoiser follows the classifier.} With the classifier defined, we can split up the mixed denoiser. By applying Bayes' rule, the mixed denoiser is:
\begin{equation}\label{eq:blend}
m(x_0 \mid x_t) = c(x_t)\, q(x_0 \mid x_t) + \big(1 - c(x_t)\big)\, p(x_0 \mid x_t).
\end{equation}
In effect, the mixed denoiser first classifies a noised sequence and then denoises it as a source sequence with probability $c(x_t)$ and as a target sequence with probability $1-c(x_t)$. What remains is what it outputs on a target sequence, and which sequences it receives.

\noindent\textbf{Fact 2: distinct domains and a target-only start.} Fact 1 tells us how the mixed denoiser completes any sequence. To turn it into the probability of generating source text, we need two more things: what it outputs and what it receives. First, the mixed denoiser outputs a source sequence with probability exactly $c(x_t)$. This follows from \eqref{eq:blend}, because the domains are distinct, $p(x_0)\,q(x_0) = 0$, so only the source denoiser can produce a source sequence. Second, the mixed denoiser always receives a noised target sequence, $x_{t_{\max}} \sim p(x_{t_{\max}})$, because only the target denoiser acts above the cutoff. Together, RefineMix generates a source sequence with probability $\mathbb{E}_{x \sim p(x_{t_{\max}})}[c(x)]$, which Theorem~\ref{thm:safe} expresses through the overlap $C$.

\begin{restatable}[RefineMix generates domain-pure sequences]{theorem}{thmsafe}\label{thm:safe}
Let the domains be distinct, $p(x_0)\,q(x_0) = 0$ for every $x_0$, let $A = \{x_0 : q(x_0) > 0\}$ be the set of source sequences, and let RefineMix use the denoiser \eqref{eq:refinemix-denoiser}. Define the \emph{overlap}
\begin{equation}\label{eq:overlap}
C(t) \coloneqq \mathbb{E}_{x \sim q(x_t)}\big[\, 1 - c(x) \,\big],
\end{equation}
the average probability that the classifier assigns to the target when shown a noised source sequence. Then, under either kernel, masked or uniform, the sequence $\hat x_0$ that RefineMix generates from pure noise is a source sequence with probability
\begin{equation}\label{eq:thrm}
\mathbb{P}(\hat x_0 \in A)
\;=\; \frac{\lambda}{1-\lambda}\, C(t_{\max})
\;\le\; \frac{\lambda}{1-\lambda}\, C(t')
\qquad \text{for every } t' \ge t_{\max}.
\end{equation}
\end{restatable}

\noindent\textbf{Reading the theorem.} The theorem reduces the cost of mixing to a single quantity, the overlap, and to know which cutoffs are safe, we need to know how it grows with noise. The overlap $C(t)$ measures how often the classifier mistakes a noised source sequence for a target one. In Figure~\ref{fig:main-figure}a, it is the purple overlap of the two noised domains. Since noising only removes evidence and never adds it, this overlap only grows with $t$ (Lemma~\ref{lem:monotone-collision}). Hence, any level $t' \ge t_{\max}$ also bounds the cost. In one sentence: as long as the classifier can still tell the domains apart at the cutoff, RefineMix rarely generates samples from the wrong domain. The proof uses only that both domains share a Markov forward process that ends in a data-independent prior, so the result does not depend on discreteness. For text, the classifier stays sharp until almost every token is masked (Intuition~III), which leaves a large range of safe cutoffs. Figure~\ref{fig:balacc} shows this in practice: a trained classifier, which can only be less accurate than the optimal one, separates the domains almost up to $t = 1$. The theorem tells us which cutoffs are safe, but not which one is best.

\subsubsection{What Mixing Can Gain} \label{sec:method-tradeoff}
In this section, we show what the benefit of mixing is for the estimated denoiser, compared to training on the target data alone. Since the target denoiser is already exact at the population level, we move to the empirical setting, where the denoiser is learned from finite data. For the source to help, the two domains must share something. Since the domains are distinct, no sequence of one domain occurs in the other, so what they share cannot be the sequences themselves. Instead, they share how content is written out: in text, for example, both domains follow the same rules of syntax, agreement, and subword spelling.

\noindent\textbf{Latent content.} To make precise what the domains share, we let each sequence arise in two steps: a domain first picks a \emph{content} $v$, and a \emph{realization} $r(x_0 \mid v)$, which is shared across the domains.
\begin{assumption}[Shared realization]\label{ass:shared}
Each clean sequence $x_0$ determines its content $v$, that is, what the text is about, and both domains realize content in tokens in the same way: $p(x_0) = \sum_v p(v)\, r(x_0 \mid v)$ and $q(x_0) = \sum_v q(v)\, r(x_0 \mid v)$.
\end{assumption}
Hence the domains differ only in the content they use, and since they are distinct, $p(v)$ and $q(v)$ have disjoint supports. Next, we show how this split appears in the denoiser.

\noindent\textbf{The denoiser splits.} To see which part of the denoiser the source can improve, we apply the latent content to it. The noised realization $r(x_t \mid v)$ is also shared between the domains, because the forward process is shared. By Bayes' rule, the denoiser splits into two factors,
\begin{equation}\label{eq:split}
p(x_0 \mid x_t) = \sum_v \underbrace{p(v \mid x_t)}_{\text{specific}}\; \underbrace{r(x_0 \mid x_t, v)}_{\text{shared}}, %
\end{equation} 
The same holds for $q$ and $m$. The two factors decide the prediction at different noise levels. The first, \emph{specific} factor $p(v \mid x_t)$ infers the content from the noised sequence, and it dominates at high noise, because there almost nothing is visible, $p(v \mid x_t) \to p(v)$ as $t \to 1$, and the model has to fall back on the content distribution of the domain. The second, \emph{shared} factor $r(x_0 \mid x_t, v)$ writes out the clean sequence given the content, and it is all that matters at low noise, because there the visible tokens already reveal the content, $p(v \mid x_t) \to \delta_{v(x_0)}$ as $t \to 0$, and the model only has to fill in the missing tokens. Since this factor is the same for both domains, the source can help to learn it.

\noindent\textbf{What each model learns.} We have now seen at which noise level each factor decides the prediction, and we turn to how each factor is learned. By the chain rule of the KL divergence (Appendix~\ref{app:gain-proofs}), the error of a trained denoiser splits into the error of the specific factor and the error of the shared factor. The specific factor can only be learned from the $N_T$ target samples. In contrast, the shared factor is the same for both domains, so below the cutoff it is learned from both corpora. We therefore assume that the shared factor learned from the source carries over to the target.
\begin{assumption}[Rate and transfer]\label{ass:rate}
Let the shared factor be learned from $N$ samples, where source and target samples count equally. Then its error at level $t$ satisfies $e_N(t) = e(t)/N + o(1/N)$ for a nonnegative, integrable error per sample $e$, where
\[
e_N(t) \coloneqq \mathbb{E}\, \mathbb{E}_{x_t \sim p(x_t)}\, \mathbb{E}_{v \sim p(v \mid x_t)}\, D_{\mathrm{KL}}\big(r(\cdot \mid x_t, v) \,\|\, \hat r_N(\cdot \mid x_t, v)\big).
\]
\end{assumption}
Assumption~\ref{ass:rate} lets source and target samples count equally. This is an idealization that Assumption~\ref{ass:shared} does not imply: in practice, what the model learns from the source may carry over only in part, and the gain is then smaller. Since Algorithm~\ref{alg:refinemix} uses both domains below the cutoff, we need to know how many equally weighted samples it learns the shared factor from. Weighting each target sample by $(1-\lambda)/N_T$ and each source sample by $\lambda/N_S$, the effective sample size of~\citet{kish1965survey} gives $N_{\mathrm{eff}} \coloneqq \big[(1-\lambda)^2/N_T + \lambda^2/N_S\big]^{-1}$, which equals $N_T + N_S$ for $\lambda = \rho \coloneqq N_S/(N_T+N_S)$.

Having seen where the source helps, we now compare the two trained models, target-only training and RefineMix. For this, we need a measure of how good a model is, the gain, and the price of mixing; as the measure, we use the NELBO on the target. Writing $\ell_\theta(x_0, t)$ for the level-$t$ term of the loss \eqref{equ_bg:loss}, and $\hat\theta_0$ and $\hat\theta_{t_{\max}}$ for the models trained by Algorithm~\ref{alg:refinemix} at cutoff $0$ (target only) and at cutoff $t_{\max}$ (RefineMix), the target NELBO and the gain are
\begin{equation}\label{eq:delta}
L_p(\theta) \coloneqq \int_0^1 \mathbb{E}_{x_0 \sim p}\, \ell_\theta(x_0, t)\, dt, \qquad \Delta(t_{\max}) \coloneqq \mathbb{E}\big[L_p(\hat\theta_0)\big] - \mathbb{E}\big[L_p(\hat\theta_{t_{\max}})\big],
\end{equation}
where the expectation is over the training samples, so that $\Delta > 0$ exactly when RefineMix beats target-only training. The price is given by $B(t_{\max}) \coloneqq \mathbb{E}_{x \sim p(x_{t_{\max}})}\big[D_{\mathrm{KL}}\big(p(v \mid x) \,\|\, m(v \mid x)\big)\big]$, the bias that the source introduces into the specific factor; it is governed by the same classifier as Theorem~\ref{thm:safe} (Appendix~\ref{app:gain-proofs}). With the gain and the price defined, we can state how the two trade off.
\begin{theorem}[Shared realization learned against price paid]\label{thm:gain}
Let the domains be distinct, let Assumptions~\ref{ass:shared} and~\ref{ass:rate} hold, and use either kernel. Then
\begin{equation}\label{eq:gain}
\Delta(t_{\max}) \;=\; \underbrace{\Big(\frac{1}{N_T} - \frac{1}{N_{\mathrm{eff}}}\Big) \int_0^{t_{\max}} e(t)\, dt}_{\text{gain}} \;-\; \underbrace{B(t_{\max})}_{\text{price}} \;+\; \text{h.o.},
\end{equation}
where h.o.\ collects terms of higher order in $1/N_T$. Moreover, the best cutoff $t^{\star}$, the largest maximizer of the leading terms, is nonincreasing in $N_T$.
\end{theorem}
\noindent\textbf{Reading the theorem.} The theorem states under the assumptions that the gain of RefineMix is a trade-off between the shared realization that is learned from the source and the bias that the source introduces into the content. The two terms depend on the cutoff in opposite ways. The first term accrues at every admitted level and scales with $1/N_T - 1/N_{\mathrm{eff}}$, which measures how badly the target needs the source. So, it is large when $N_T$ is small. In contrast, the bias stays small while the visible tokens still reveal the content, and it grows only at high noise, where the content becomes hard to infer. Hence the best cutoff sits where the two meet, and it moves in two ways. It rises as the target corpus shrinks. We also expect it to fall for domains that share fewer writing rules.

\noindent\textbf{Reading the best cutoff from two models.}\label{sec:method-diagnostic} The theorem tells us how the best cutoff moves: it rises as the target shrinks, and falls when the domains share fewer rules. It does not tell us where the best cutoff is, because $e(t)$ cannot be measured directly. However, we can read the trade-off directly from two models: the fully pooled model and the target-only model. On held-out target text, we compare their level-$t$ losses through the \emph{loss gap} $g(t) \coloneqq \bar\ell_m(t) - \bar\ell_p(t)$.
Where $g < 0$, the source teaches more than it costs, and where $g > 0$, the price has taken over. If each noise level were learned separately, then $\Delta(t_{\max}) = -\int_0^{t_{\max}} g(t)\, dt$, and the best cutoff would be the level at which $g$ turns positive. In our experiments, the loss gap indeed points to a good cutoff, but it requires training the pooled model to convergence. Alternatively, one can find the cutoff by sweeping over several values.

\begin{table}[t]
\centering
\small
\setlength{\tabcolsep}{3pt}
\begin{tabular}{l r r r r r}
\toprule
\textbf{Pair (src $\to$ tgt)} & \textbf{Tgt.\ tokens} & \textbf{Fine-tune} & \textbf{Data Mixing} & \textbf{Conditional} & \textbf{RefineMix ($t_\text{max}$)} \\
\midrule
Math $\to$ Stat               & 12.96\,M &  19.88\,$\pm$\,0.05 & 19.61\,$\pm$\,0.01 & 19.10\,$\pm$\,0.00 & \textbf{18.06}$\pm$0.09  (0.7) \\
Physics $\to$ Hep-Ex         &  4.71\,M &  7.460\,$\pm$\,0.019 &  7.17\,$\pm$\,0.01 &  7.09\,$\pm$\,0.01 & \textbf{6.833} $\pm$0.04 (0.8) \\
\midrule
cc\_news $\to$ Hep-Ex         &  4.71\,M &  7.790\,$\pm$\,0.027 &  8.05\,$\pm$\,0.01 & 11.28\,$\pm$\,0.01 & \textbf{7.310}\,$\pm$\,0.021 (0.8) \\
cc\_news $\to$ Textiles      &  1.17\,M &  20.15\,$\pm$\,0.11 & 19.73\,$\pm$\,0.12 & 30.12\,$\pm$\,1.19 & \textbf{19.37}\,$\pm$\,0.10 (0.9) \\
Patents-G $\to$ AG News      &  5.67\,M &  27.42\,$\pm$\,0.19 & 31.55\,$\pm$\,0.16 & 29.55\,$\pm$\,0.04 & 27.41\,$\pm$\,0.03 (0.2) \\
\bottomrule
\end{tabular}
\caption{\textbf{RefineMix beats the baselines when adapting a source-pretrained model.} Target validation perplexity (lower is better) starting from a source-pretrained checkpoint (masked kernel). Fine-tuning uses the target data alone, data mixing adds the source at every noise level, the conditional model adds it with a domain label, and RefineMix adds it only below $t_\text{max}$. Fine-tuning and RefineMix use the learning rate that is best for fine-tuning (three seeds); data mixing and the conditional model are earlier runs at $3\times10^{-4}$ (three seeds).}
\label{tab:two-stage}
\end{table}

\begin{figure}[t]
  \centering
  \begin{subfigure}[t]{0.3\linewidth}
    \centering
    \includegraphics[width=\linewidth]{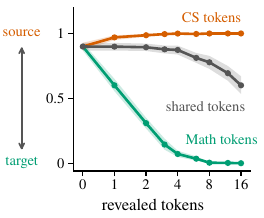}
    \caption{Few tokens settle the domain.}
    \label{fig:purity}
  \end{subfigure}
  \hfill
  \begin{subfigure}[t]{0.3\linewidth}
    \centering
    \includegraphics[width=\linewidth]{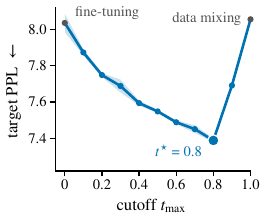}
    \caption{Best cutoff lies in between.}
    \label{fig:tradeoff-cutoff}
  \end{subfigure}
  \hfill
  \begin{subfigure}[t]{0.3\linewidth}
    \centering
    \includegraphics[width=\linewidth]{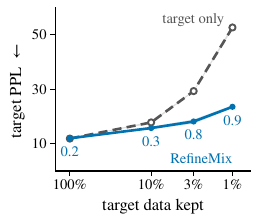}
    \caption{Scarcer target, larger gain.}
    \label{fig:tradeoff-scarcity}
  \end{subfigure}
   \caption{\textbf{Testing the theory} (masked kernel, three seeds). \textbf{(a)} CS $\to$ Math. A fully mixed model completes a noised sequence with $n$ revealed tokens from CS, Math, or both. \textbf{(b)} cc\_news $\to$ Hep-Ex, from the source checkpoint. Target validation perplexity against the cutoff, whose endpoints are fine-tuning and data mixing. \textbf{(c)} CS $\to$ Math, from scratch. Target validation perplexity against the fraction of the target corpus kept, for target-only training and RefineMix at its best cutoff.}
  \label{fig:ablations}
\end{figure}

\begin{table}[t]
\centering
\small
\setlength{\tabcolsep}{5pt}
\begin{tabular}{llcccc}
\toprule
& & \multicolumn{2}{c}{Fine-tuning on target} & \multicolumn{2}{c}{Target + source} \\
\cmidrule(lr){3-4}\cmidrule(lr){5-6}
Model & Target & Full & LoRA & Data mixing & RefineMix ($t_\text{max}$) \\
\midrule
MDLM 130M & Lean/mathlib & 3.327{\scriptsize$\pm$0.008} & 3.775{\scriptsize$\pm$0.007} & 3.439  & \textbf{3.262}{\scriptsize$\pm$0.002} (0.2) \\
\citep{sahoo2024simpleeffectivemaskeddiffusion} & Rare disease & 15.61{\scriptsize$\pm$0.02} & 18.859{\scriptsize$\pm$0.117} & 15.770  & \textbf{15.00}{\scriptsize$\pm$0.02} (0.4) \\
 & ChemRxiv & 18.69{\scriptsize$\pm$0.04} & 25.248{\scriptsize$\pm$0.164} & 20.260  & \textbf{17.97}{\scriptsize$\pm$0.01} (0.4) \\
\midrule
SEDD 424M & Lean/mathlib & 3.045{\scriptsize$\pm$0.003} & 3.314{\scriptsize$\pm$0.013} & 3.147  & \textbf{3.018}{\scriptsize$\pm$0.03}  (0.3) \\
\citep{lou2024discrete} & Rare disease & 13.85{\scriptsize$\pm$0.02} & 15.473{\scriptsize$\pm$0.009} & 13.720  & \textbf{13.14}{\scriptsize$\pm$0.007} (0.5) \\
 & ChemRxiv & 16.39{\scriptsize$\pm$0.01} & 19.659{\scriptsize$\pm$0.126} & 17.340  & \textbf{15.77}{\scriptsize$\pm$0.01} (0.4) \\
\midrule
SMDM 1.1B & Lean/mathlib & 2.868{\scriptsize$\pm$0.002} & 3.180{\scriptsize$\pm$0.014} & 2.941  & \textbf{2.849}{\scriptsize$\pm$0.02}  (0.7) \\
\citep{nie2025scaling} & Rare disease & 7.980{\scriptsize$\pm$0.002} & 8.629{\scriptsize$\pm$0.031} & 8.808  & \textbf{7.728}{\scriptsize$\pm$0.01}  (0.6) \\
 & ChemRxiv & 10.50{\scriptsize$\pm$0.03} & 11.541{\scriptsize$\pm$0.004} & 11.570  & \textbf{10.06}{\scriptsize$\pm$0.01} (0.5) \\
\midrule
DiffuLLaMA 7B$^{*}$ & Lean/mathlib & \textbf{2.941}{\scriptsize$\pm$0.003} & -- & 3.512  & 2.955 {\scriptsize$\pm$0.001} (0.4) \\
\citep{gong2025scaling} & Rare disease & 8.323{\scriptsize$\pm$0.003} & -- & 9.203  & \textbf{8.181} {\scriptsize$\pm$0.02} (0.5) \\
 & ChemRxiv & 10.75{\scriptsize$\pm$0.01} & -- & 12.730  & \textbf{10.49}{\scriptsize$\pm$0.00} (0.4) \\
\bottomrule
\end{tabular}
\caption{\textbf{RefineMix adapts pretrained diffusion language models to scarce domains.} Target validation perplexity (lower is better). The first two columns fine-tune on the target alone, training all parameters (full) or LoRA adapters (rank $16$); the last two also use the source, at every noise level (data mixing) or only below $t_\text{max}$ (RefineMix). Full fine-tuning and RefineMix use the learning rate that is best for fine-tuning (seed $1$); LoRA and data mixing are earlier runs at $3\times10^{-4}$. $^{*}$DiffuLLaMA trains LoRA (rank $64$) in every column, so its separate LoRA column is empty.}
\label{tab:paper}
\end{table}

\begin{table}[t]
\centering
\small
\setlength{\tabcolsep}{5pt}
\begin{tabular}{l ccc|c}
\toprule
& Family & Folds & Novel & All three \\
\midrule
Real peroxiredoxins (ceiling) & 94.0{\scriptsize$\pm$0.8} & 98.2{\scriptsize$\pm$0.9} & 78.1{\scriptsize$\pm$6.8} & 75.1{\scriptsize$\pm$5.7} \\
\midrule
Fine-tuning & 51.8{\scriptsize$\pm$5.0} & 37.5{\scriptsize$\pm$4.6} & 99.0{\scriptsize$\pm$0.3} & 35.8{\scriptsize$\pm$4.6} \\
Data mixing & 36.2{\scriptsize$\pm$2.2} & 35.7{\scriptsize$\pm$2.2} & 98.7{\scriptsize$\pm$0.2} & 32.9{\scriptsize$\pm$2.2} \\
RefineMix ($t_{\max} = 0.7$) & \textbf{79.8}{\scriptsize$\pm$2.4} & \textbf{74.3}{\scriptsize$\pm$2.9} & 89.8{\scriptsize$\pm$1.8} & \textbf{63.3}{\scriptsize$\pm$2.0} \\
\midrule
Generic UniRef50 (floor) & 0.0 & 65.1 & 100.0 & 0.0 \\
\bottomrule
\end{tabular}
\caption{\textbf{RefineMix writes new, folding members of a protein family from $197$ examples.} Percentage of sequences generated by DPLM-650M that pass each test of Section~\ref{sec:experiments}; \emph{All three} counts sequences that pass all tests at once. Real held-out peroxiredoxins and generic proteins mark the ceiling and the floor. Mean $\pm$ standard deviation over three train--test splits, each with three seeds.}
\label{tab:protein}
\end{table}

\section{Experiments}\label{sec:experiments}
In the previous section, we introduced our algorithm RefineMix. It only changes which data the model sees at each noise level, so it can be implemented as a data loader. Hence it works on top of any architecture, at any stage of training, and with either kernel, masked or uniform. In this section, we show that RefineMix works in practice and test the predictions of the theory. First, we compare RefineMix to the standard baselines across domain pairs. Second, we run ablations in controlled settings: we show how much revealed target content keeps sampling on target (Theorem~\ref{thm:safe}), and how performance changes over the cutoff and the target size (Theorem~\ref{thm:gain}). Third, we adapt publicly available diffusion language models to new domains. Finally, we adapt a pretrained protein diffusion model to generate a specific protein family from only a few examples.

 \noindent\textbf{Experimental setup.} For the comparison to the baselines and the ablations, we build source--target pairs from four corpora. Within arXiv abstracts~\citep{arxiv_org_submitters_2024}, we use the categories Math, Stat, Physics, CS, and Hep-Ex (high-energy physics, experiment). Across corpora, we use CC-News~\citep{hamborg2017newsplease} and AG News~\citep{zhang2015character}, which contain news articles, and two sections of BIGPATENT~\citep{sharma2019bigpatent}: Patents-G (physics, section G) and Textiles (textiles and paper, section D). Table~\ref{tab:data-sizes} lists the corpus sizes.  Every run uses the same $139$M-parameter DiT denoiser, trained identically. We evaluate the models by their target validation perplexity, because it is the quantity that the gain $\Delta$ measures (Theorem~\ref{thm:gain}). Since a mixed model can reach a high likelihood on the target domain and still write source text, we add two more metrics in Appendix~\ref{app:finetune}. There we report the generation frontier, the generative perplexity of the samples against their diversity over a sweep of sampling temperatures, and the fraction of generated samples that belong to the target domain. We report three seeds per setting, we choose the cutoff sweep start with the loss gap $g(t)$ of Section~\ref{sec:method-diagnostic}, select the cutoff with the lowest target validation perplexity, and set $\lambda = 0.9$.

\noindent\textbf{Comparison to baselines.} With the data and the model set up, we turn to the setting that is common in practice: a model is pretrained on the source, and one wants to fine-tune it on a smaller target corpus. We show that RefineMix beats this recipe, as well as the other standard baselines. For each pair and seed, we start from the same source-pretrained checkpoint and train a second stage in four ways: on the target data alone (fine-tuning), on both corpora at every noise level (data mixing), on both corpora with a domain label (conditional), and with RefineMix. 

Table~\ref{tab:two-stage} shows that RefineMix reaches the lowest target perplexity on four of the five pairs, even though it runs at the learning rate that is best for fine-tuning (Appendix~\ref{app:lr}). The gain over fine-tuning is largest for the close arXiv pairs, about $8$--$9\%$, and smallest for Patents-G~$\to$~AG News, where the two match. The best cutoff varies with the distance between the domains: it lies between $0.7$ and $0.9$ for close pairs and drops to $0.2$ for Patents-G~$\to$~AG News. RefineMix beats fine-tuning because fine-tuning on the small target corpus overwrites what pretraining learned, whereas RefineMix keeps this knowledge below the cutoff (Figure~\ref{fig:forgetting}). Finally, Appendix~\ref{app:finetune} shows that the samples of RefineMix stay on target and reach a better generation frontier than fine-tuning.

\noindent\textbf{Ablations.} The comparison shows that RefineMix works, and we now test the theory with one experiment per prediction. \textbf{First}, a few revealed target tokens should keep sampling on target (Theorem~\ref{thm:safe}). To test this, we take a fully mixed model ($t_{\max} = 1$) on CS~$\to$~Math with $1\%$ of the target, and reveal $n$ tokens of a noised sequence before sampling starts. With nothing revealed, the model writes source text at its training rate $\lambda$, whereas about four revealed target tokens turn it to the target domain, and shared tokens change nothing (Figure~\ref{fig:purity}; Appendix~\ref{app:bound} tests the theorem). \textbf{Second}, admitting the source at more noise levels helps, because the model learns more of what the domains share, and hurts, because the source pulls the model toward its own domain (Theorem~\ref{thm:gain}). The harm grows only at high noise and eventually outgrows the help, the best cutoff lies between fine-tuning and data mixing. The perplexity is indeed lowest at $t^\star = 0.8$, and every cutoff in this sweep beats both endpoints, fine-tuning ($t_{\max} = 0$) and data mixing ($t_{\max} = 1$) (Figure~\ref{fig:tradeoff-cutoff}). \textbf{Third}, the help from the source should matter more when the target is small (Theorem~\ref{thm:gain}). On CS~$\to$~Math, as we shrink the target from $100\%$ to $1\%$, RefineMix goes from a small improvement to more than halving the target-only perplexity (Figure~\ref{fig:tradeoff-scarcity}). Appendix~\ref{app:uniform} repeats these experiments under the uniform kernel (Figure~\ref{fig:ablations-uniform}), and Appendix~\ref{app:beyond} tests Theorem~\ref{thm:safe} on images under Gaussian noise.

\noindent\textbf{Adapting pretrained models.} The experiments above confirm the theory in controlled settings with our own $139$M model. In practice, however, one rarely trains a base model and instead adapts a publicly available one to a desired domain. We therefore adapt four pretrained diffusion language models, from $130$M to $7$B parameters, to three scarce targets, using fractions of their pretraining corpora as the source: the Analysis subset of the Lean~4 library mathlib~\citep{mathlib2020lean}, the subset of PubMed Central case reports on rare diseases~\citep{pubmedcasereports}, and the full set of ChemRxiv abstracts~\citep{laionchemrxiv}. The pretrained models may have seen the targets during their original pretraining, but all methods start from the same weights, so this cannot explain the differences between them. Since RefineMix only changes the data loader, it applies to every model unchanged. Both methods use the learning rate that is best for fine-tuning, so RefineMix receives no learning-rate tuning of its own (Appendix~\ref{app:lr}). RefineMix improves on fine-tuning by $2$--$5\%$ on Rare disease and ChemRxiv for all four models, and matches it on Lean/mathlib, where the differences lie between $-2.0\%$ and $+0.5\%$ (Table~\ref{tab:paper}). LoRA, which trains fewer parameters, does worse than full fine-tuning, consistent with the observation that LoRA is suboptimal for diffusion language models~\citep{wang2026nara}. Appendix~\ref{app:pretrained-samples} compares the generation frontiers of RefineMix and the baselines.

\noindent\textbf{Proteins.} So far, all our data domains were text. The same data scarcity problem arises for proteins. Imagine for example, trying to model a certain family of proteins with a limited number of known protein members. To study this problem, we test whether RefineMix can be used to finetune the large pretrained protein model DPLM-650M~\citep{wang2024dplm} to produce new members of peroxiredoxins, a family of antioxidant enzymes. UniProt lists $83{,}921$ peroxiredoxins, but only $274$ of them are curated by hand~\citep{uniprot2023}, and after removing duplicates, $237$ remain, so the scarcity here is real rather than simulated. We use $197$ of them as the target and generic UniRef50 proteins as the source, from which we remove every peroxiredoxin cluster. DPLM-650M is pretrained on UniRef50 and has likely seen peroxiredoxins, but all methods share this checkpoint. We split by sequence cluster, so that no test sequence is identical to a training sequence, which would otherwise reward memorization. After training, we generate $256$ sequences with the checkpoint of lowest validation perplexity. A generated sequence counts only if it passes three tests at once: it belongs to the family (HMMER~\citep{eddy2011hmmer} matches it to a Pfam profile of peroxiredoxins~\citep{mistry2021pfam}), it folds (ESMFold pLDDT $> 70$~\citep{lin2023esmfold}), and it is new (below $90\%$ identity to every training sequence). RefineMix produces roughly two times as many such sequences as fine-tuning, which overfits quickly on such a small dataset (Table~\ref{tab:protein}). Data mixing, in contrast, fails already in the Family column: by Theorem~\ref{thm:safe}, its samples come from the source at rate $\lambda$, and since we removed every peroxiredoxin cluster from the source, these samples cannot pass the family test.

\section{Conclusion}\label{sec:conclusion}
Discrete noise, unlike Gaussian noise on images, does not erase the domain of a text: the domains remain separable. Therefore, out-of-domain data belongs only at low noise, whereas continuous diffusion can also use imperfect data at high noise. We therefore introduced RefineMix, which admits source data only below a cutoff. The method only changes the data loader, so it works on any architecture, at any stage of training, and with either kernel. Our theory shows why sampling stays in the target domain as long as a classifier can still tell the domains apart at the cutoff, and that the best cutoff balances what the source teaches against the bias it introduces at high noise. Across domain pairs, RefineMix beats data mixing, conditional training, LoRA, and matches or beats full fine-tuning, and it also improves pretrained models from $130$M to $7$B parameters and a protein family. Its main limitations are that the cutoff must still be chosen, by a sweep or from two trained models, and that our theory describes an ideal model with fully effective transfer rather than the trained network. This leaves room for future work: a cheap diagnostic that predicts the cutoff before training, and for cutoffs chosen per sample rather than one for the whole corpus.

\subsection*{AI Use Statement}
In this work, we used generative AI tools for assistance with paper writing, editing figures, developing the mathematical theorems and proofs, implementing parts of the method and conducting the experiments. We have \textbf{not} used generative AI tools for conceptualization of the project, results interpretation and developing the structure and the narrative of the paper. We have reviewed all AI-assisted work. In particular, we reviewed the code, the proofs, all the paper edits and parsing of the results of the experiments. The authors fully understand and take ownership of all the materials presented in this work. We take responsibility for the final content of this work, including text, claims or artifacts produced with the aid of generative AI.

\subsection*{Acknowledgements}
Giannis Daras is supported by a Jane Street Research Collaboration award and compute grants from the TPU Research Cloud, Lambda Research and Thinking Machines.

\bibliography{iclr2027_conference}

@misc{schiff2025simpleguidance,
      title={Simple Guidance Mechanisms for Discrete Diffusion Models}, 
      author={Yair Schiff and Subham Sekhar Sahoo and Hao Phung and Guanghan Wang and Sam Boshar and Hugo Dalla-torre and Bernardo P. de Almeida and Alexander Rush and Thomas Pierrot and Volodymyr Kuleshov},
      year={2025},
      eprint={2412.10193},
      archivePrefix={arXiv},
      primaryClass={cs.LG},
      url={https://arxiv.org/abs/2412.10193}, 
}

@misc{shi2025simplifiedgeneralizedmaskeddiffusion,
      title={Simplified and Generalized Masked Diffusion for Discrete Data}, 
      author={Jiaxin Shi and Kehang Han and Zhe Wang and Arnaud Doucet and Michalis K. Titsias},
      year={2025},
      eprint={2406.04329},
      archivePrefix={arXiv},
      primaryClass={cs.LG},
      url={https://arxiv.org/abs/2406.04329}, 
}

@misc{austin2023structureddenoisingdiffusionmodels,
      title={Structured Denoising Diffusion Models in Discrete State-Spaces}, 
      author={Jacob Austin and Daniel D. Johnson and Jonathan Ho and Daniel Tarlow and Rianne van den Berg},
      year={2023},
      eprint={2107.03006},
      archivePrefix={arXiv},
      primaryClass={cs.LG},
      url={https://arxiv.org/abs/2107.03006}, 
}

@misc{sahoo2024simpleeffectivemaskeddiffusion,
      title={Simple and Effective Masked Diffusion Language Models}, 
      author={Subham Sekhar Sahoo and Marianne Arriola and Yair Schiff and Aaron Gokaslan and Edgar Marroquin and Justin T Chiu and Alexander Rush and Volodymyr Kuleshov},
      year={2024},
      eprint={2406.07524},
      archivePrefix={arXiv},
      primaryClass={cs.CL},
      url={https://arxiv.org/abs/2406.07524}, 
}

@article{wei2026ambient,
  title={Ambient Diffusion Policy: Imitation Learning from Suboptimal Data in Robotics},
  author={Wei, Adam and Pfaff, Nicholas and Cohn, Thomas and Day{\i}, Arif Kerem and Daskalakis, Constantinos and Daras, Giannis and Tedrake, Russ},
  journal={arXiv preprint arXiv:2606.12365},
  year={2026}
}

@misc{wang2025finetuningdiscretediffusionmodels,
      title={Fine-Tuning Discrete Diffusion Models via Reward Optimization with Applications to DNA and Protein Design}, 
      author={Chenyu Wang and Masatoshi Uehara and Yichun He and Amy Wang and Tommaso Biancalani and Avantika Lal and Tommi Jaakkola and Sergey Levine and Hanchen Wang and Aviv Regev},
      year={2025},
      eprint={2410.13643},
      archivePrefix={arXiv},
      primaryClass={cs.LG},
      url={https://arxiv.org/abs/2410.13643}, 
}

@misc{peebles2023scalablediffusionmodelstransformers,
      title={Scalable Diffusion Models with Transformers}, 
      author={William Peebles and Saining Xie},
      year={2023},
      eprint={2212.09748},
      archivePrefix={arXiv},
      primaryClass={cs.CV},
      url={https://arxiv.org/abs/2212.09748}, 
}

@misc{arxiv_org_submitters_2024,
	title={arXiv Dataset},
	url={https://www.kaggle.com/dsv/7548853},
	DOI={10.34740/KAGGLE/DSV/7548853},
	publisher={Kaggle},
	author={arXiv.org submitters},
	year={2024}
}

@article{labs2025mercury,
  title   = {Mercury: Ultra-Fast Language Models Based on Diffusion},
  author  = {Inception Labs and Khanna, Samar and Kharbanda, Siddhant and Li, Shufan and Varma, Harshit and Wang, Eric and Birnbaum, Sawyer and Luo, Ziyang and Miraoui, Yanis and Palrecha, Akash and others},
  journal = {arXiv preprint arXiv:2506.17298},
  year    = {2025}
}

@article{lee2025genmol,
  title   = {{GenMol}: A Drug Discovery Generalist with Discrete Diffusion},
  author  = {Lee, Seul and Kreis, Karsten and Veccham, Srimukh Prasad and Liu, Meng and Reidenbach, Danny and Peng, Yuxing and Paliwal, Saee and Nie, Weili and Vahdat, Arash},
  journal = {arXiv preprint arXiv:2501.06158},
  year    = {2025}
}

@article{liu2023generative,
  title   = {Generative Diffusion Models on Graphs: Methods and Applications},
  author  = {Liu, Chengyi and Fan, Wenqi and Liu, Yunqing and Li, Jiatong and Li, Hang and Liu, Hui and Tang, Jiliang and Li, Qing},
  journal = {arXiv preprint arXiv:2302.02591},
  year    = {2023}
}

@misc{lou2024discrete,
  title         = {Discrete Diffusion Modeling by Estimating the Ratios of the Data Distribution},
  author        = {Lou, Aaron and Meng, Chenlin and Ermon, Stefano},
  year          = {2024},
  eprint        = {2310.16834},
  archivePrefix = {arXiv},
  url           = {https://arxiv.org/abs/2310.16834}
}

@misc{ou2025your,
  title         = {Your Absorbing Discrete Diffusion Secretly Models the Conditional Distributions of Clean Data},
  author        = {Ou, Jingyang and Nie, Shen and Xue, Kaiwen and Zhu, Fengqi and Sun, Jiacheng and Li, Zhenguo and Li, Chongxuan},
  year          = {2025},
  eprint        = {2406.03736},
  archivePrefix = {arXiv},
  url           = {https://arxiv.org/abs/2406.03736}
}

@inproceedings{sahoo2025diffusionduality,
  title     = {The Diffusion Duality},
  author    = {Sahoo, Subham Sekhar and Deschenaux, Justin and Gokaslan, Aaron and Wang, Guanghan and Chiu, Justin T. and Kuleshov, Volodymyr},
  booktitle = {Forty-second International Conference on Machine Learning},
  year      = {2025},
  url       = {https://openreview.net/forum?id=9P9Y8FOSOk}
}

@article{sahoo2025esoteric,
  title   = {Esoteric Language Models},
  author  = {Sahoo, Subham Sekhar and Yang, Zhihan and Akhauri, Yash and Liu, Johnna and Singh, Deepansha and Cheng, Zhoujun and Liu, Zhengzhong and Xing, Eric and Thickstun, John and Vahdat, Arash},
  journal = {arXiv preprint arXiv:2506.01928},
  year    = {2025}
}

@article{song2025seed,
  title   = {Seed Diffusion: A Large-Scale Diffusion Language Model with High-Speed Inference},
  author  = {Song, Yuxuan and Zhang, Zheng and Luo, Cheng and Gao, Pengyang and Xia, Fan and Luo, Hao and Li, Zheng and Yang, Yuehang and Yu, Hongli and Qu, Xingwei and others},
  journal = {arXiv preprint arXiv:2508.02193},
  year    = {2025}
}

@misc{daras2025ambientdiffusionomnitraining,
      title={Ambient Diffusion Omni: Training Good Models with Bad Data}, 
      author={Giannis Daras and Adrian Rodriguez-Munoz and Adam Klivans and Antonio Torralba and Constantinos Daskalakis},
      year={2025},
      eprint={2506.10038},
      archivePrefix={arXiv},
      primaryClass={cs.GR},
      url={https://arxiv.org/abs/2506.10038}, 
}

@misc{kaplan2020scalinglawsneurallanguage,
      title={Scaling Laws for Neural Language Models}, 
      author={Jared Kaplan and Sam McCandlish and Tom Henighan and Tom B. Brown and Benjamin Chess and Rewon Child and Scott Gray and Alec Radford and Jeffrey Wu and Dario Amodei},
      year={2020},
      eprint={2001.08361},
      archivePrefix={arXiv},
      primaryClass={cs.LG},
      url={https://arxiv.org/abs/2001.08361}, 
}

@misc{hoffmann2022trainingcomputeoptimallargelanguage,
      title={Training Compute-Optimal Large Language Models}, 
      author={Jordan Hoffmann and Sebastian Borgeaud and Arthur Mensch and Elena Buchatskaya and Trevor Cai and Eliza Rutherford and Diego de Las Casas and Lisa Anne Hendricks and Johannes Welbl and Aidan Clark and Tom Hennigan and Eric Noland and Katie Millican and George van den Driessche and Bogdan Damoc and Aurelia Guy and Simon Osindero and Karen Simonyan and Erich Elsen and Jack W. Rae and Oriol Vinyals and Laurent Sifre},
      year={2022},
      eprint={2203.15556},
      archivePrefix={arXiv},
      primaryClass={cs.CL},
      url={https://arxiv.org/abs/2203.15556}, 
}

@misc{albalak2023efficientonlinedatamixing,
      title={Efficient Online Data Mixing For Language Model Pre-Training}, 
      author={Alon Albalak and Liangming Pan and Colin Raffel and William Yang Wang},
      year={2023},
      eprint={2312.02406},
      archivePrefix={arXiv},
      primaryClass={cs.CL},
      url={https://arxiv.org/abs/2312.02406}, 
}

@misc{kleutgens2026guidedtransferlearningdiscrete,
      title={Guided Transfer Learning for Discrete Diffusion Models}, 
      author={Julian Kleutgens and Claudio Battiloro and Lingkai Kong and Benjamin Grewe and Francesca Dominici and Mauricio Tec},
      year={2026},
      eprint={2512.10877},
      archivePrefix={arXiv},
      primaryClass={cs.LG},
      url={https://arxiv.org/abs/2512.10877}, 
}

@article{saerens2002adjusting,
  author  = {Saerens, Marco and Latinne, Patrice and Decaestecker, Christine},
  title   = {Adjusting the Outputs of a Classifier to New a Priori Probabilities: A Simple Procedure},
  journal = {Neural Computation},
  volume  = {14},
  number  = {1},
  pages   = {21--41},
  year    = {2002},
  doi     = {10.1162/089976602753284446}
}

@inproceedings{hamborg2017newsplease,
  title     = {news-please: A Generic News Crawler and Extractor},
  author    = {Hamborg, Felix and Meuschke, Norman and Breitinger, Corinna and Gipp, Bela},
  booktitle = {Proceedings of the 15th International Symposium of Information Science},
  pages     = {218--223},
  year      = {2017}
}

@inproceedings{sharma2019bigpatent,
  title     = {{BIGPATENT}: A Large-Scale Dataset for Abstractive and Coherent Summarization},
  author    = {Sharma, Eva and Li, Chen and Wang, Lu},
  booktitle = {Proceedings of the 57th Annual Meeting of the Association for Computational Linguistics},
  pages     = {2204--2213},
  year      = {2019}
}

@inproceedings{zhang2015character,
  title     = {Character-level Convolutional Networks for Text Classification},
  author    = {Zhang, Xiang and Zhao, Junbo and LeCun, Yann},
  booktitle = {Advances in Neural Information Processing Systems},
  volume    = {28},
  year      = {2015}
}

@article{ni2025diffusion,
  title={Diffusion language models are super data learners},
  author={Ni, Jinjie and Liu, Qian and Dou, Longxu and Du, Chao and Wang, Zili and Yan, Hang and Pang, Tianyu and Shieh, Michael Qizhe},
  journal={arXiv preprint arXiv:2511.03276},
  year={2025}
}

@article{gao2025makes,
  title={What Makes Diffusion Language Models Super Data Learners?},
  author={Gao, Zitian and Luo, Haoming and Chen, Lynx and Liu, Jason Klein and Tao, Ran and Zhou, Joey and Dai, Bryan},
  journal={arXiv preprint arXiv:2510.04071},
  year={2025}
}

@article{prabhudesai2026diffusion,
  title={Diffusion beats autoregressive in data-constrained settings},
  author={Prabhudesai, Mihir and Wu, Mengning and Zadeh, Amir and Fragkiadaki, Katerina and Pathak, Deepak},
  journal={Advances in Neural Information Processing Systems},
  volume={38},
  pages={10581--10606},
  year={2026}
}

@inproceedings{ambient_diffusion,
 author = {Daras, Giannis and Shah, Kulin and Dagan, Yuval and Gollakota, Aravind and Dimakis, Alex and Klivans, Adam},
 booktitle = {Advances in Neural Information Processing Systems},
 editor = {A. Oh and T. Naumann and A. Globerson and K. Saenko and M. Hardt and S. Levine},
 pages = {288--313},
 publisher = {Curran Associates, Inc.},
 title = {Ambient Diffusion: Learning Clean Distributions from Corrupted Data},
 url = {https://proceedings.neurips.cc/paper_files/paper/2023/file/012af729c5d14d279581fc8a5db975a1-Paper-Conference.pdf},
 volume = {36},
 year = {2023}
}

@inproceedings{
ambient_dps,
title={Ambient Diffusion Posterior Sampling: Solving Inverse Problems with Diffusion Models Trained on Corrupted Data},
author={Asad Aali and Giannis Daras and Brett Levac and Sidharth Kumar and Alex Dimakis and Jon Tamir},
booktitle={The Thirteenth International Conference on Learning Representations},
year={2025},
url={https://openreview.net/forum?id=qeXcMutEZY}
}

@inproceedings{consistent_diffusion_meets_tweedie,
author = {Daras, Giannis and Dimakis, Alexandros G. and Daskalakis, Constantinos},
title = {Consistent diffusion meets tweedie: training exact ambient diffusion models with noisy data},
year = {2024},
publisher = {JMLR.org},
booktitle = {Proceedings of the 41st International Conference on Machine Learning},
articleno = {401},
numpages = {18},
location = {Vienna, Austria},
series = {ICML'24}
}

@inproceedings{
daras2025how,
title={How Much is a  Noisy Image Worth? Data Scaling Laws for Ambient Diffusion.},
author={Giannis Daras and Yeshwanth Cherapanamjeri and Constantinos Costis Daskalakis},
booktitle={The Thirteenth International Conference on Learning Representations},
year={2025},
url={https://openreview.net/forum?id=qZwtPEw2qN}
}

@article{dataloops,
  title={Ambient Dataloops: Generative Models for Dataset Refinement},
  author={Rodr{\'\i}guez-Mu{\~n}oz, Adri{\'a}n and Daspit, William and Klivans, Adam and Torralba, Antonio and Daskalakis, Constantinos and Daras, Giannis},
  journal={arXiv preprint arXiv:2601.15417},
  year={2026}
}

@article{hosseintabar2025diffem,
  title={Diffem: Learning from corrupted data with diffusion models via expectation maximization},
  author={Hosseintabar, Danial and Chen, Fan and Daras, Giannis and Torralba, Antonio and Daskalakis, Constantinos},
  journal={arXiv preprint arXiv:2510.12691},
  year={2025}
}

@article{sfbdomni,
  title={SFBD-OMNI: Bridge models for lossy measurement restoration with limited clean samples},
  author={Lu, Haoye and Yu, Yaoliang and Ho, Darren},
  journal={arXiv preprint arXiv:2512.17051},
  year={2025}
}

@article{park2025measurement,
  title={Measurement score-based diffusion model},
  author={Park, Chicago Y and Shoushtari, Shirin and An, Hongyu and Kamilov, Ulugbek S},
  journal={arXiv preprint arXiv:2505.11853},
  year={2025}
}

@article{modi2025generative,
  title={Generative Modeling from Black-box Corruptions via Self-Consistent Stochastic Interpolants},
  author={Modi, Chirag and Han, Jiequn and Vanden-Eijnden, Eric and Bruna, Joan},
  journal={arXiv preprint arXiv:2512.10857},
  year={2025}
}

@article{kawar2023gsure,
  title={Gsure-based diffusion model training with corrupted data},
  author={Kawar, Bahjat and Elata, Noam and Michaeli, Tomer and Elad, Michael},
  journal={arXiv preprint arXiv:2305.13128},
  year={2023}
}

@inproceedings{aali2023solving,
  title={Solving inverse problems with score-based generative priors learned from noisy data},
  author={Aali, Asad and Arvinte, Marius and Kumar, Sidharth and Tamir, Jonathan I},
  booktitle={2023 57th Asilomar Conference on Signals, Systems, and Computers},
  pages={837--843},
  year={2023},
  organization={IEEE}
}

@article{matsuzaki2026score,
  title={SCoRe: Clean Image Generation from Diffusion Models Trained on Noisy Images},
  author={Matsuzaki, Yuta and Uchida, Seiichi and Takezaki, Shumpei},
  journal={arXiv preprint arXiv:2604.09436},
  year={2026}
}

@inproceedings{
sfbd,
title={Stochastic Forward{\textendash}Backward Deconvolution: Training Diffusion Models with Finite Noisy Datasets},
author={Haoye Lu and Qifan Wu and Yaoliang Yu},
booktitle={Forty-second International Conference on Machine Learning},
year={2025},
url={https://openreview.net/forum?id=WrWqv3mpQx}
}

@inproceedings{
rozet2024learning,
title={Learning Diffusion Priors from Observations by Expectation Maximization},
author={Fran{\c{c}}ois Rozet and G{\'e}r{\^o}me Andry and Francois Lanusse and Gilles Louppe},
booktitle={The Thirty-eighth Annual Conference on Neural Information Processing Systems},
year={2024},
url={https://openreview.net/forum?id=7v88Fh6iSM}
}

@inproceedings{em_diffusion,
author = {Bai, Weimin and Wang, Yifei and Chen, Wenzheng and Sun, He},
title = {An expectation-maximization algorithm for training clean diffusion models from corrupted observations},
year = {2025},
isbn = {9798331314385},
publisher = {Curran Associates Inc.},
address = {Red Hook, NY, USA},
booktitle = {Proceedings of the 38th International Conference on Neural Information Processing Systems},
articleno = {613},
numpages = {25},
location = {Vancouver, BC, Canada},
series = {NIPS '24}
}

@article{zhang2025restoration,
  title={Restoration score distillation: From corrupted diffusion pretraining to one-step high-quality generation},
  author={Zhang, Yasi and Chen, Tianyu and Wang, Zhendong and Wu, Ying Nian and Zhou, Mingyuan and Leong, Oscar},
  journal={arXiv preprint arXiv:2505.13377},
  year={2025}
}

@article{tewari2023diffusion,
  title={Diffusion with forward models: Solving stochastic inverse problems without direct supervision},
  author={Tewari, Ayush and Yin, Tianwei and Cazenavette, George and Rezchikov, Semon and Tenenbaum, Josh and Durand, Fr{\'e}do and Freeman, Bill and Sitzmann, Vincent},
  journal={Advances in Neural Information Processing Systems},
  volume={36},
  pages={12349--12362},
  year={2023}
}

@inproceedings{gong2025scaling,
  title     = {Scaling Diffusion Language Models via Adaptation from Autoregressive Models},
  author    = {Gong, Shansan and Agarwal, Shivam and Zhang, Yizhe and Ye, Jiacheng and Zheng, Lin and Li, Mukai and An, Chenxin and Zhao, Peilin and Bi, Wei and Han, Jiawei and Peng, Hao and Kong, Lingpeng},
  booktitle = {The Thirteenth International Conference on Learning Representations},
  year      = {2025},
  url       = {https://openreview.net/forum?id=j1tSLYKwg8}
}

@article{ye2025dream,
  title   = {Dream 7B: Diffusion Large Language Models},
  author  = {Ye, Jiacheng and Xie, Zhihui and Zheng, Lin and Gao, Jiahui and Wu, Zirui and Jiang, Xin and Li, Zhenguo and Kong, Lingpeng},
  journal = {arXiv preprint arXiv:2508.15487},
  year    = {2025}
}

@article{zhao2025d1,
  title   = {d1: Scaling Reasoning in Diffusion Large Language Models via Reinforcement Learning},
  author  = {Zhao, Siyan and Gupta, Devaansh and Zheng, Qinqing and Grover, Aditya},
  journal = {arXiv preprint arXiv:2504.12216},
  year    = {2025}
}

@inproceedings{nisonoff2025unlocking,
  title     = {Unlocking Guidance for Discrete State-Space Diffusion and Flow Models},
  author    = {Nisonoff, Hunter and Xiong, Junhao and Allenspach, Stephan and Listgarten, Jennifer},
  booktitle = {The Thirteenth International Conference on Learning Representations},
  year      = {2025},
  url       = {https://openreview.net/forum?id=XsgHl54yO7}
}

@article{bendavid2010theory,
  title   = {A theory of learning from different domains},
  author  = {Ben-David, Shai and Blitzer, John and Crammer, Koby and Kulesza, Alex and Pereira, Fernando and Vaughan, Jennifer Wortman},
  journal = {Machine Learning},
  volume  = {79},
  number  = {1--2},
  pages   = {151--175},
  year    = {2010},
  doi     = {10.1007/s10994-009-5152-4}
}

@inproceedings{ouyang2024transfer,
  title     = {Transfer Learning for Diffusion Models},
  author    = {Ouyang, Yidong and Xie, Liyan and Zha, Hongyuan and Cheng, Guang},
  booktitle = {Advances in Neural Information Processing Systems},
  volume    = {37},
  year      = {2024},
  note      = {arXiv:2405.16876}
}

@article{ho2022classifier,
  title   = {Classifier-Free Diffusion Guidance},
  author  = {Ho, Jonathan and Salimans, Tim},
  journal = {arXiv preprint arXiv:2207.12598},
  year    = {2022}
}

@inproceedings{xie2023doremi,
  title     = {{DoReMi}: Optimizing Data Mixtures Speeds Up Language Model Pretraining},
  author    = {Xie, Sang Michael and Pham, Hieu and Dong, Xuanyi and Du, Nan and Liu, Hanxiao and Lu, Yifeng and Liang, Percy and Le, Quoc V. and Ma, Tengyu and Yu, Adams Wei},
  booktitle = {Advances in Neural Information Processing Systems},
  volume    = {36},
  year      = {2023}
}

@inproceedings{liu2024regmix,
  title     = {{RegMix}: Data Mixture as Regression for Language Model Pre-training},
  author    = {Liu, Qian and Zheng, Xiaosen and Muennighoff, Niklas and Zeng, Guangtao and Dou, Longxu and Pang, Tianyu and Jiang, Jing and Lin, Min},
  booktitle = {The Thirteenth International Conference on Learning Representations},
  year      = {2025},
  url       = {https://openreview.net/forum?id=5BjQOUXq7i}
}

@book{kish1965survey,
  title     = {Survey Sampling},
  author    = {Kish, Leslie},
  year      = {1965},
  publisher = {John Wiley \& Sons},
  address   = {New York}
}

@inproceedings{wang2026nara,
  title     = {{NaRA}: Noise-Aware {LoRA} for Parameter-Efficient Fine-Tuning of Diffusion {LLMs}},
  author    = {Wang, Shuaidi and Zhuang, Zhan and Huang, Ruping and Zhang, Yu},
  booktitle = {Proceedings of the 43rd International Conference on Machine Learning},
  series    = {Proceedings of Machine Learning Research},
  volume    = {306},
  year      = {2026}
}

@inproceedings{nie2025scaling,
  title     = {Scaling up Masked Diffusion Models on Text},
  author    = {Nie, Shen and Zhu, Fengqi and Du, Chao and Pang, Tianyu and Liu, Qian and Zeng, Guangtao and Lin, Min and Li, Chongxuan},
  booktitle = {International Conference on Learning Representations},
  year      = {2025}
}

@inproceedings{wang2024dplm,
  title     = {Diffusion Language Models Are Versatile Protein Learners},
  author    = {Wang, Xinyou and Zheng, Zaixiang and Ye, Fei and Xue, Dongyu and Huang, Shujian and Gu, Quanquan},
  booktitle = {Proceedings of the 41st International Conference on Machine Learning},
  series    = {Proceedings of Machine Learning Research},
  volume    = {235},
  pages     = {52309--52333},
  publisher = {PMLR},
  year      = {2024},
  url       = {https://proceedings.mlr.press/v235/wang24ct.html}
}

@article{uniprot2023,
  title   = {{UniProt}: the Universal Protein Knowledgebase in 2023},
  author  = {{The UniProt Consortium}},
  journal = {Nucleic Acids Research},
  volume  = {51},
  number  = {D1},
  pages   = {D523--D531},
  year    = {2023},
  doi     = {10.1093/nar/gkac1052}
}

@article{mistry2021pfam,
  title   = {{Pfam}: The protein families database in 2021},
  author  = {Mistry, Jaina and Chuguransky, Sara and Williams, Lowri and Qureshi, Matloob and Salazar, Gustavo A. and Sonnhammer, Erik L. L. and Tosatto, Silvio C. E. and Paladin, Lisanna and Raj, Shriya and Richardson, Lorna J. and Finn, Robert D. and Bateman, Alex},
  journal = {Nucleic Acids Research},
  volume  = {49},
  number  = {D1},
  pages   = {D412--D419},
  year    = {2021},
  doi     = {10.1093/nar/gkaa913}
}

@article{lin2023esmfold,
  title   = {Evolutionary-scale prediction of atomic-level protein structure with a language model},
  author  = {Lin, Zeming and Akin, Halil and Rao, Roshan and Hie, Brian and Zhu, Zhongkai and Lu, Wenting and Smetanin, Nikita and Verkuil, Robert and Kabeli, Ori and Shmueli, Yaniv and dos Santos Costa, Allan and Fazel-Zarandi, Maryam and Sercu, Tom and Candido, Salvatore and Rives, Alexander},
  journal = {Science},
  volume  = {379},
  number  = {6637},
  pages   = {1123--1130},
  year    = {2023},
  doi     = {10.1126/science.ade2574}
}

@article{eddy2011hmmer,
  title   = {Accelerated Profile {HMM} Searches},
  author  = {Eddy, Sean R.},
  journal = {PLoS Computational Biology},
  volume  = {7},
  number  = {10},
  pages   = {e1002195},
  year    = {2011}
}

@inproceedings{mathlib2020lean,
  title     = {The Lean Mathematical Library},
  author    = {{The mathlib Community}},
  booktitle = {Proceedings of the 9th ACM SIGPLAN International Conference on Certified Programs and Proofs (CPP)},
  pages     = {367--381},
  year      = {2020},
  doi       = {10.1145/3372885.3373824}
}

@inproceedings{rolnick2019experience,
  title     = {Experience Replay for Continual Learning},
  author    = {Rolnick, David and Ahuja, Arun and Schwarz, Jonathan and Lillicrap, Timothy P. and Wayne, Gregory},
  booktitle = {Advances in Neural Information Processing Systems},
  volume    = {32},
  year      = {2019}
}

@inproceedings{ouyang2022training,
  title     = {Training Language Models to Follow Instructions with Human Feedback},
  author    = {Ouyang, Long and Wu, Jeff and Jiang, Xu and Almeida, Diogo and Wainwright, Carroll L. and Mishkin, Pamela and Zhang, Chong and Agarwal, Sandhini and Slama, Katarina and Ray, Alex and others},
  booktitle = {Advances in Neural Information Processing Systems},
  volume    = {35},
  year      = {2022}
}

@misc{pubmedcasereports,
  title        = {{PubMed} Case Reports},
  author       = {{awinml}},
  howpublished = {\url{https://huggingface.co/datasets/awinml/pubmed_case_reports}},
  note         = {Hugging Face dataset, accessed 2026},
  year         = {2026}
}

@misc{orphadata,
  title        = {{Orphadata}: Free Access Data from {Orphanet}},
  author       = {{INSERM}},
  howpublished = {\url{https://www.orphadata.com}},
  note         = {Rare disease nomenclature and epidemiology files, accessed 2026},
  year         = {2026}
}

@misc{laionchemrxiv,
  title        = {{ChemRxiv} Metadata},
  author       = {{LAION}},
  howpublished = {\url{https://huggingface.co/datasets/laion/chemrxiv_metadata}},
  note         = {Snapshot of 15 September 2024},
  year         = {2024}
}

@article{biroli2024dynamical,
  title   = {Dynamical regimes of diffusion models},
  author  = {Biroli, Giulio and Bonnaire, Tony and de Bortoli, Valentin and M{\'e}zard, Marc},
  journal = {Nature Communications},
  volume  = {15},
  pages   = {9957},
  year    = {2024},
  doi     = {10.1038/s41467-024-54281-3}
}

@article{sclocchi2025phase,
  title   = {A phase transition in diffusion models reveals the hierarchical nature of data},
  author  = {Sclocchi, Antonio and Favero, Alessandro and Wyart, Matthieu},
  journal = {Proceedings of the National Academy of Sciences},
  volume  = {122},
  number  = {1},
  pages   = {e2408799121},
  year    = {2025},
  doi     = {10.1073/pnas.2408799121}
}
\bibliographystyle{iclr2027_conference}

\appendix
\clearpage
\addtocontents{toc}{\protect\setcounter{tocdepth}{2}}  %
\renewcommand{\contentsname}{Contents and Appendix}
\tableofcontents
\clearpage  %

\section{Experimental details}\label{app:details}
This section lists what is needed to rerun the experiments of Section~\ref{sec:experiments}: the setup shared by all text experiments, and what each experiment adds to it.

\paragraph{Data.} Text is tokenized with \texttt{bert-base-uncased} ($|\mathcal V| = 30{,}522$), whose \texttt{[MASK]} token serves as the absorbing state. The documents of a corpus are concatenated with \texttt{[CLS]} between them and cut into non-overlapping $512$-token blocks (\texttt{[CLS]}, $510$ text tokens, \texttt{[CLS]}) without padding, after a $90/10$ train--validation split by document (seed $42$). The data is listed in Table~\ref{tab:data-sizes} with the resulting sizes. 
\begin{table}[h]
\centering
\small
\setlength{\tabcolsep}{5pt}
\begin{tabular}{l l r}
\toprule
Corpus & Tokenizer & Tokens \\
\midrule
arXiv Math     & \texttt{bert-base-uncased} & $103.2$\,M \\
arXiv Stat     & (our $139$M DiT) & $14.4$\,M \\
arXiv Physics  &  & $46.5$\,M \\
arXiv Hep-Ex   &  & $5.3$\,M \\
arXiv CS       &  & $193.1$\,M \\
CC-News        &  & $373.5$\,M \\
Patents-G      &  & $35.9$\,M \\
Textiles       &  & $1.3$\,M \\
AG News        &  & $6.3$\,M \\
\midrule
OpenWebText    & GPT-2 & $113.3$\,M \\
Lean/mathlib   & (MDLM, SEDD) & $2.8$\,M \\
Rare disease   &  & $3.0$\,M \\
ChemRxiv       &  & $7.6$\,M \\
\midrule
SlimPajama-6B  & LLaMA-2 & $133.9$\,M \\
Lean/mathlib   & (SMDM, DiffuLLaMA) & $2.6$\,M \\
Rare disease   &  & $3.5$\,M \\
ChemRxiv       &  & $8.7$\,M \\
\bottomrule
\end{tabular}
\caption{\textbf{Corpora}, in tokens of the listed tokenizer. The upper block serves our $139$M model (Table~\ref{tab:two-stage} and Figure~\ref{fig:ablations}), the lower two the pretrained models of Table~\ref{tab:paper}. Of the two web corpora we use only one parquet file each of OpenWebText ($100{,}173$ documents) and of \texttt{DKYoon/SlimPajama-6B}. Figure~\ref{fig:tradeoff-scarcity} keeps $100\%$ down to $1\%$ of the Math training data.}
\label{tab:data-sizes}
\end{table}

\paragraph{Model and training.} All text models use the same DiT denoiser~\citep{peebles2023scalablediffusionmodelstransformers} with $139.3$M parameters: $12$ blocks, hidden size $768$, $12$ attention heads, MLP width $3{,}072$, rotary position embeddings, dropout $0.1$, and adaLN conditioning on the noise level. The masked kernel uses the SUBS parameterization~\citep{sahoo2024simpleeffectivemaskeddiffusion}, a log-linear schedule ($\epsilon = 10^{-3}$), and the continuous-time NELBO with $t \sim \mathcal U[10^{-3}, 1]$. We train with AdamW ($\beta = (0.9, 0.999)$, $\epsilon = 10^{-8}$, no weight decay) at a constant learning rate chosen per setting by the sweep of Appendix~\ref{app:lr} ($3\times10^{-4}$ in our earlier runs) after $2{,}500$ linear warm-up steps, with gradient clipping at $1.0$, bf16, and $256$ blocks per step ($64$ per micro-batch, $4$ accumulation steps) on one NVIDIA RTX PRO 6000 or B200. We keep an exponential moving average (EMA) of the weights with decay $0.9999$ and use it for all validation and sampling. Data mixing and RefineMix share one data loader: each item pairs a target block with a source block, draws $t$ per item, and trains on the source block with probability $\lambda = 0.9$ if $t \le t_{\max}$ and on the target block otherwise, with equal loss weights; data mixing is $t_{\max} = 1$. Every setting uses seeds $1$, $2$, and $3$.

\paragraph{Evaluation.} We report the target validation perplexity, the exponential of the per-token NELBO on the full target validation split. Validation during training uses a single pass and only selects checkpoints.

\paragraph{Comparison to baselines (Table~\ref{tab:two-stage}).} In the first stage, we train one source-only model per pair from scratch, for at most $30$k steps or $12$ hours, and keep the checkpoint with the lowest source validation perplexity. The second stage loads the weights and the EMA and restarts the optimizer and the warm-up. The conditional model adds a zero-initialized embedding for three labels (source, target, none) to the noise-level conditioning, trains on the labelled union of both corpora with label dropout $0.1$, and is sampled with the target label. To choose $t^\star$, we start where the loss gap $g(t)$ turns positive and move in steps of $0.1$ with seed $1$ until both neighbours are worse; seeds $2$ and $3$ then use this cutoff. Figures~\ref{fig:tradeoff-cutoff}, \ref{fig:frontier-finetune}, \ref{fig:lossgap-twostage}, \ref{fig:forgetting}, \ref{fig:pretrained-frontier} and~\ref{fig:pretrained-ontarget} come from runs at a constant learning rate of $3\times10^{-4}$.

\paragraph{Revealed tokens (Figure~\ref{fig:purity}).} We train three fully mixed models on CS~$\to$~Math from scratch. To form the token groups, we score candidate words, whole words of at least three letters that are not stopwords and occur at least $100$ times in at least $10$ validation blocks, by their weighted log-odds between the two domains under an informative Dirichlet prior. Math and CS tokens are the $300$ candidates with the largest $z$-scores towards each domain, and shared tokens are the $300$ candidates above the $60$th frequency percentile with the smallest log-odds difference. For each group, we take $256$ Math validation blocks with at least $16$ tokens of the group, mask every position except the first $n$ of these $16$ tokens, with $n \in \{0, 1, 2, 3, 4, 6, 8, 12, 16\}$, and sample a completion with $1{,}000$ ancestral steps at temperature $1$, starting at $t = (512 - n)/512$. A domain classifier scores each completion. Each point averages $256$ completions per model over the three models.

\paragraph{Cutoff sweep (Figure~\ref{fig:tradeoff-cutoff}).} We run the second stage of cc\_news~$\to$~Hep-Ex with the protocol above at every $t_{\max} \in \{0, 0.1, \dots, 1.0\}$ with three seeds. The uniform-kernel version (Appendix~\ref{app:uniform}) uses a DUO model~\citep{sahoo2025diffusionduality}.

\paragraph{Target size (Figure~\ref{fig:tradeoff-scarcity}).} We train on CS~$\to$~Math from scratch under both kernels, with all CS blocks and $100\%$, $10\%$, $3\%$, or $1\%$ of the Math training blocks. Target-only training and RefineMix train for exactly $30$k steps without early stopping, with three seeds each. We do the same for the uniform kernel.

\paragraph{Pretrained models (Table~\ref{tab:paper}).} We adapt MDLM ($130$M) and SEDD ($424$M), pretrained on OpenWebText, and SMDM ($1.1$B) and DiffuLLaMA ($7$B), pretrained on SlimPajama, with a public sample of the respective pretraining corpus as the source (Table~\ref{tab:data-sizes}). Lean/mathlib is the Analysis directory of mathlib ($340$ files)~\citep{mathlib2020lean}; Rare disease is the $980$ PubMed Central case reports that name a disorder Orphanet classes as rare~\citep{pubmedcasereports,orphadata}; ChemRxiv is $26{,}234$ titles and abstracts from the snapshot of September 2024~\citep{laionchemrxiv}. We drop every source document that shares a $12$-word sequence with a target document. Each target is split $90/10$ by document. Every run starts from the public weights with a fresh optimizer and EMA and otherwise follows the training setup above. DiffuLLaMA trains LoRA adapters (rank $64$, $\alpha = 128$) in every column, and the LoRA column uses rank $16$, $\alpha = 32$. Each run stops when its target validation perplexity stops improving or at its step cap, and we keep the best checkpoint. We search the cutoff in steps of $0.1$ starting where $g(t)$ turns positive until both neighbours are worse, and report the validation perplexity with three seeds.

\paragraph{Proteins (Table~\ref{tab:protein}).} The target is the peroxiredoxin family: $237$ reviewed Swiss-Prot sequences (median $198$ residues) in $118$ UniRef50 clusters. We split by whole cluster, three times with disjoint held-out sets, each with $197$ to $200$ training and $37$ to $40$ held-out sequences (about $39$k training tokens). The source is UniRef50, capped at $500{,}000$ sequences drawn round-robin across shards, since the dump is sorted by length, with all $102$ peroxiredoxin clusters removed. Every run starts from DPLM-650M, uses the ESM-2 vocabulary ($|\mathcal V| = 33$) with one sequence per $512$-token block, three folds, and three seeds. For the cutoff, we sweep over $0.3$, $0.5$ and $0.7$. From the checkpoint with the lowest validation perplexity, we generate $256$ sequences in $256$ denoising steps, with lengths drawn from the training distribution. A sequence counts as Family if \texttt{hmmsearch} hits PF00578, PF08534, or PF10417 at the profile's gathering threshold ($26.0$, $26.6$, $21.1$ bits), Folds if its mean ESMFold pLDDT exceeds $70$, and Novel if MMseqs2 finds no training sequence of its fold above $90\%$ identity; All three is the per-sample intersection.

\subsection{Results under the uniform kernel}\label{app:uniform}
The main text uses the masked kernel. Here we repeat the comparison on cc\_news~$\to$~Hep-Ex and the three ablations under the uniform kernel, and the conclusions are unchanged. RefineMix again reaches the lowest target perplexity (Table~\ref{tab:uniform}), a few revealed target tokens keep sampling on target, and the best cutoff lies between the two baselines and rises as the target shrinks (Figure~\ref{fig:ablations-uniform}).

\begin{table}[H]
\centering
\small
\setlength{\tabcolsep}{3pt}
\begin{tabular}{l r r r r r}
\toprule
\textbf{Pair (src $\to$ tgt)} & \textbf{Tgt.\ tokens} & \textbf{Fine-tune} & \textbf{Data mixing} & \textbf{Conditional} & \textbf{RefineMix ($t_\text{max}$)} \\
\midrule
cc\_news $\to$ Hep-Ex & 4.71\,M & 8.49\,$\pm$\,0.03 & 8.76\,$\pm$\,0.02 & 12.29$^{\dagger}$ & \textbf{7.96\,$\pm$\,0.01} (0.7) \\
\bottomrule
\end{tabular}
\caption{\textbf{Comparison to baselines under the uniform kernel.} Target validation perplexity, same setting as Table~\ref{tab:two-stage}. Mean $\pm$ standard deviation over three seeds; $^{\dagger}$one seed.}
\label{tab:uniform}
\end{table}

\begin{figure}[H]
  \centering
  \begin{subfigure}[t]{0.32\linewidth}
    \centering
    \includegraphics[width=\linewidth]{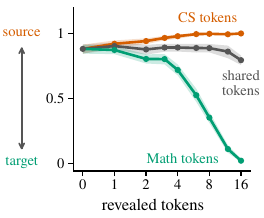}
    \caption{More tokens settle the domain.}
    \label{fig:uniform-purity}
  \end{subfigure}
  \hfill
  \begin{subfigure}[t]{0.32\linewidth}
    \centering
    \includegraphics[width=\linewidth]{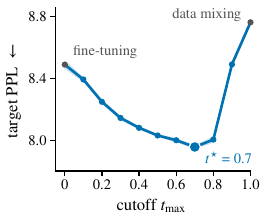}
    \caption{The best cutoff lies in between.}
    \label{fig:uniform-cutoff}
  \end{subfigure}
  \hfill
  \begin{subfigure}[t]{0.32\linewidth}
    \centering
    \includegraphics[width=\linewidth]{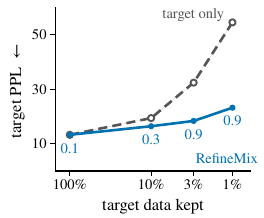}
    \caption{Scarcer target, larger gain.}
    \label{fig:uniform-scarcity}
  \end{subfigure}
  \caption{\textbf{The three predictions of the theory under the uniform kernel} (three seeds). Same experiments as Figure~\ref{fig:ablations}; in (b), the source checkpoint is a DUO model of the same size, trained on same datasets.}
  \label{fig:ablations-uniform}
\end{figure}

\subsection{Ablation of the source rate}\label{app:source-rate}
So far, we have chosen the cutoff and fixed the source rate at $\lambda = 0.9$. Besides the cutoff, Algorithm~\ref{alg:refinemix} has one more hyperparameter: the source rate $\lambda$, the probability that a training item below the cutoff is a source sample. To check whether $\lambda = 0.9$ is a good choice, we sweep $\lambda$ on CS~$\to$~Math with $1\%$ of the target, trained from scratch at cutoff $t_\text{max} = 0.9$, and measure both the target validation perplexity and the on-target rate of the samples at temperature $1.0$ (Table~\ref{tab:source-rate}). On this pair, $\lambda \in \{0.9, 0.95\}$ gives the lowest target validation perplexity, and the samples of RefineMix stay on target at every source rate; we use $\lambda = 0.9$ throughout.

We run the same sweep for data mixing, which is the endpoint $t_{\max} = 1$ of RefineMix, just as fine-tuning is the endpoint $t_{\max} = 0$. Its on-target rate falls roughly as $1-\lambda$, as Theorem~\ref{thm:safe} predicts: from $0.96$ at $\lambda = 0.1$ to $0.13$ at $\lambda = 0.9$. Data mixing reaches a lower target perplexity than RefineMix at every $\lambda$, but its samples drift to the source, which perplexity on target text does not register. At the only source rate where it stays on target, $\lambda = 0.1$, its perplexity is $45.5$, about twice that of RefineMix at its best source rate. Data mixing therefore has no source rate that is both on target and competitive, whereas RefineMix stays on target at every $\lambda$.

\begin{table}[H]
\centering
\small
\setlength{\tabcolsep}{5pt}
\begin{tabular}{c cc cc}
\toprule
& \multicolumn{2}{c}{RefineMix ($t_\text{max} = 0.9$)} & \multicolumn{2}{c}{Data mixing} \\
\cmidrule(lr){2-3}\cmidrule(lr){4-5}
$\lambda$ & PPL & On-target & PPL & On-target \\
\midrule
0.1  & 46.7 & 1.00 & 45.5 & 0.96 \\
0.5  & 36.8 & 1.00 & 36.6 & 0.63 \\
0.7  & 29.6 & 0.99 & 28.8 & 0.39 \\
0.9  & 23.3 & 0.99 & 21.0 & 0.13 \\
0.95 & 22.1 & 0.98 & 18.8 & 0.03 \\
0.99 & 24.1 & 0.96 & 21.0 & 0.01 \\
\bottomrule
\end{tabular}
\caption{\textbf{Ablation of the source rate $\lambda$} (CS $\to$ Math, $1\%$ of the target, trained from scratch, masked kernel, seed $1$). Target validation perplexity (lower is better) and on-target rate, the fraction of samples that the domain classifier assigns to the target at sampling temperature $1.0$. }
\label{tab:source-rate}
\end{table}

\subsection{Tuning the learning rate}\label{app:lr}
Fine-tuning is sensitive to the learning rate, so a single fixed value could understate it. We therefore sweep the learning rate for fine-tuning and RefineMix alike, with the same early stopping as above, and report both methods in Tables~\ref{tab:two-stage} and~\ref{tab:paper} at the learning rate that is best for fine-tuning, so that RefineMix receives no learning-rate tuning of its own. For our $139$M model we sweep $\{3\times10^{-5}, 10^{-4}, 3\times10^{-4}, 10^{-3}, 3\times10^{-3}\}$; for the pretrained models we start from $\{3\times10^{-5}, 10^{-4}, 3\times10^{-4}\}$ and extend the grid towards smaller or larger rates where the best value lies at its end, down to $10^{-5}$ for SMDM and up to $3\times10^{-3}$ for MDLM, SEDD and DiffuLLaMA. RefineMix keeps its cutoff from the original search. Validation perplexities are averaged over five passes over the validation split, and the sweep uses seed $1$. Table~\ref{tab:lr} lists all runs. The best learning rate depends on the model: the smaller models prefer large rates, around $10^{-3}$, whereas SMDM prefers small ones, around $10^{-5}$.

\begin{table}[H]
\centering
\scriptsize
\setlength{\tabcolsep}{3pt}
\begin{tabular}{ll cccccc}
\toprule
Setting & Method & $10^{-5}$ & $3\!\times\!10^{-5}$ & $10^{-4}$ & $3\!\times\!10^{-4}$ & $10^{-3}$ & $3\!\times\!10^{-3}$ \\
\midrule
MDLM, Lean/mathlib & Fine-tuning & -- & 3.462 & 3.412 & 3.342 & \textbf{3.320} & 3.386 \\
 & RefineMix (0.2) & -- & 3.358 & 3.334 & 3.291 & \textbf{3.263} & -- \\
MDLM, Rare disease & Fine-tuning & -- & 15.79 & \textbf{15.61} & 15.96 & -- & -- \\
 & RefineMix (0.4) & -- & 15.17 & \textbf{14.99} & \textbf{14.99} & -- & -- \\
MDLM, ChemRxiv & Fine-tuning & -- & 18.91 & 18.83 & \textbf{18.67} & 18.73 & -- \\
 & RefineMix (0.4) & -- & 18.02 & 18.02 & \textbf{17.96} & 18.07 & -- \\
\midrule
SEDD, Lean/mathlib & Fine-tuning & -- & 3.175 & 3.134 & 3.050 & \textbf{3.043} & 3.134 \\
 & RefineMix (0.3) & -- & -- & 3.065 & 3.019 & \textbf{3.018} & -- \\
SEDD, Rare disease & Fine-tuning & -- & \textbf{13.87} & \textbf{13.87} & 14.07 & -- & -- \\
 & RefineMix (0.5) & -- & 13.30 & \textbf{13.14} & 13.26 & -- & -- \\
SEDD, ChemRxiv & Fine-tuning & -- & 16.63 & 16.48 & \textbf{16.39} & 16.59 & -- \\
 & RefineMix (0.4) & -- & -- & 15.82 & \textbf{15.76} & -- & -- \\
\midrule
SMDM, Lean/mathlib & Fine-tuning & \textbf{2.866} & 2.874 & 2.910 & 3.020 & -- & -- \\
 & RefineMix (0.7) & \textbf{2.849} & 2.858 & 2.850 & 2.965 & -- & -- \\
SMDM, Rare disease & Fine-tuning & \textbf{8.026} & 8.094 & 8.178 & 8.407 & -- & -- \\
 & RefineMix (0.6) & \textbf{7.672} & 7.703 & 7.831 & 8.103 & -- & -- \\
SMDM, ChemRxiv & Fine-tuning & 10.51 & \textbf{10.50} & 10.65 & 11.12 & -- & -- \\
 & RefineMix (0.5) & 10.09 & \textbf{10.06} & 10.19 & 10.69 & -- & -- \\
\midrule
DiffuLLaMA, Lean/mathlib & Fine-tuning & -- & 3.075 & 2.996 & 2.954 & \textbf{2.943} & 3.080 \\
 & RefineMix (0.4) & -- & 3.046 & 2.976 & \textbf{2.939} & 2.955 & -- \\
DiffuLLaMA, Rare disease & Fine-tuning & -- & 8.402 & 8.323 & \textbf{8.310} & -- & -- \\
 & RefineMix (0.5) & -- & 8.226 & 8.181 & \textbf{8.122} & -- & -- \\
DiffuLLaMA, ChemRxiv & Fine-tuning & -- & 10.86 & \textbf{10.75} & \textbf{10.75} & 10.95 & -- \\
 & RefineMix (0.4) & -- & 10.63 & \textbf{10.49} & 10.57 & 11.15 & -- \\
\midrule
Math $\to$ Stat & Fine-tuning & -- & 20.07 & 19.95 & \textbf{19.85} & 20.07 & -- \\
 & RefineMix (0.7) & -- & 19.05 & 18.37 & 18.06 & \textbf{17.81} & -- \\
Physics $\to$ Hep-Ex & Fine-tuning & -- & 7.563 & 7.483 & 7.479 & \textbf{7.470} & -- \\
 & RefineMix (0.8) & -- & 7.052 & 6.934 & 6.833 & \textbf{6.700} & 8.944 \\
cc\_news $\to$ Hep-Ex & Fine-tuning & -- & 8.139 & 7.983 & 7.854 & \textbf{7.810} & 8.035 \\
 & RefineMix (0.8) & -- & 7.851 & 7.480 & 7.375 & \textbf{7.332} & -- \\
cc\_news $\to$ Textiles & Fine-tuning & -- & 20.97 & 20.51 & 20.98 & \textbf{20.23} & 20.65 \\
 & RefineMix (0.9) & -- & 20.27 & 20.13 & 19.71 & \textbf{19.36} & 20.13 \\
Patents-G $\to$ AG News & Fine-tuning & -- & 30.44 & 29.30 & 28.24 & \textbf{27.64} & 28.70 \\
 & RefineMix (0.2) & -- & 29.49 & 28.65 & 27.96 & \textbf{27.40} & 29.39 \\
\bottomrule
\end{tabular}
\caption{\textbf{Learning-rate sweep.} Target validation perplexity (lower is better, seed $1$, averaged over five validation passes) of fine-tuning and RefineMix (cutoff in parentheses) at each learning rate, for the pretrained models of Table~\ref{tab:paper} (top four blocks) and our $139$M model from a source checkpoint (Table~\ref{tab:two-stage}, bottom block). \textbf{Bold}: best learning rate of each method. --: not run.}
\label{tab:lr}
\end{table}
\subsection{Additional results for \S\ref{sec:experiments}: RefineMix against fine-tuning}\label{app:finetune}
This section collects the two supporting results for Table~\ref{tab:two-stage}. Figure~\ref{fig:frontier-finetune} scores the generated samples of the same models, and Figure~\ref{fig:forgetting} measures what each model retains of the source checkpoint it started from. Figure~\ref{fig:balacc} shows how well a domain classifier separates the domains at each noise level. 
\begin{figure}[H]
  \centering
  \begin{subfigure}[t]{\linewidth}
    \centering
    \includegraphics[width=0.9\linewidth]{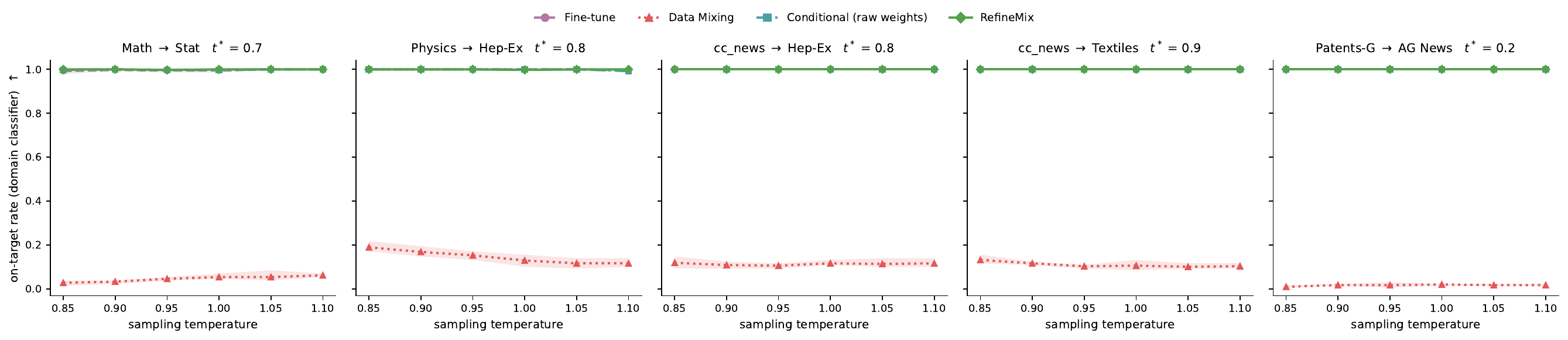}
    \caption{On-target rate against sampling temperature, all four models.}
    \label{fig:frontier-ontarget}
  \end{subfigure}
  \\[1em]
  \begin{subfigure}[t]{\linewidth}
    \centering
    \includegraphics[width=0.9\linewidth]{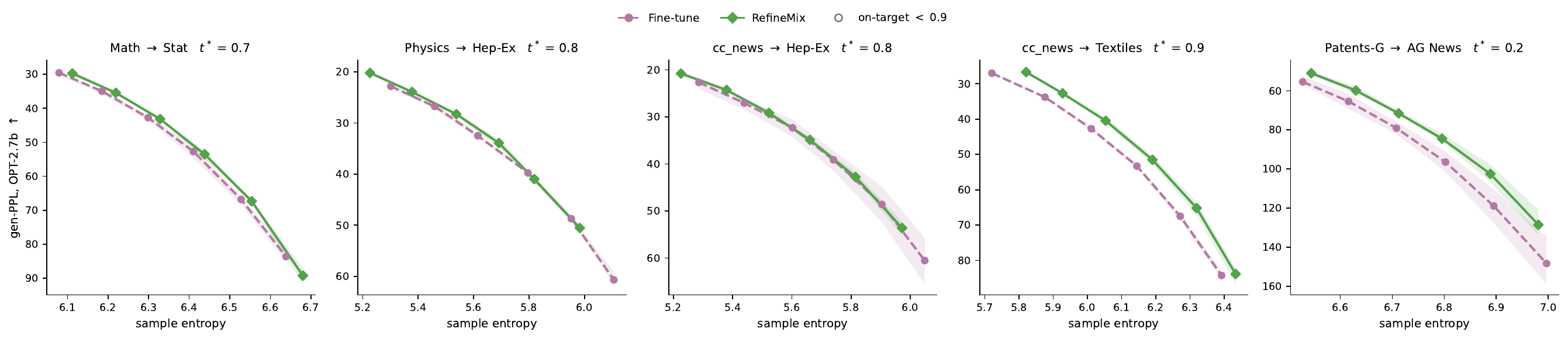}
    \caption{Generative perplexity against sample entropy, fine-tuning and RefineMix.}
    \label{fig:frontier-genppl}
  \end{subfigure}
  \caption{\textbf{Samples generated by the models of Table~\ref{tab:two-stage}.} For each of the five source--target pairs, all models start from the same source-pretrained checkpoint and are adapted to the target (masked kernel, mean over three seeds). We sample from pure noise over a sweep of sampling temperatures. \textbf{(a)} On-target rate: the fraction of samples that a domain classifier assigns to the target, against the sampling temperature, for fine-tuning, data mixing, the conditional model, and RefineMix. RefineMix stays on target at every temperature. \textbf{(b)} Generation frontier: generative perplexity of the samples, scored by OPT-2.7B, against their entropy, for fine-tuning and RefineMix. Perplexity is inverted, so up and to the right means more fluent and more diverse samples, and the dashed line marks real target text. The RefineMix frontier matches or lies above fine-tuning on every pair.}
  \label{fig:frontier-finetune}
\end{figure}

\begin{figure}[H]
  \centering
  \begin{subfigure}[t]{0.49\linewidth}
    \centering
    \includegraphics[width=\linewidth]{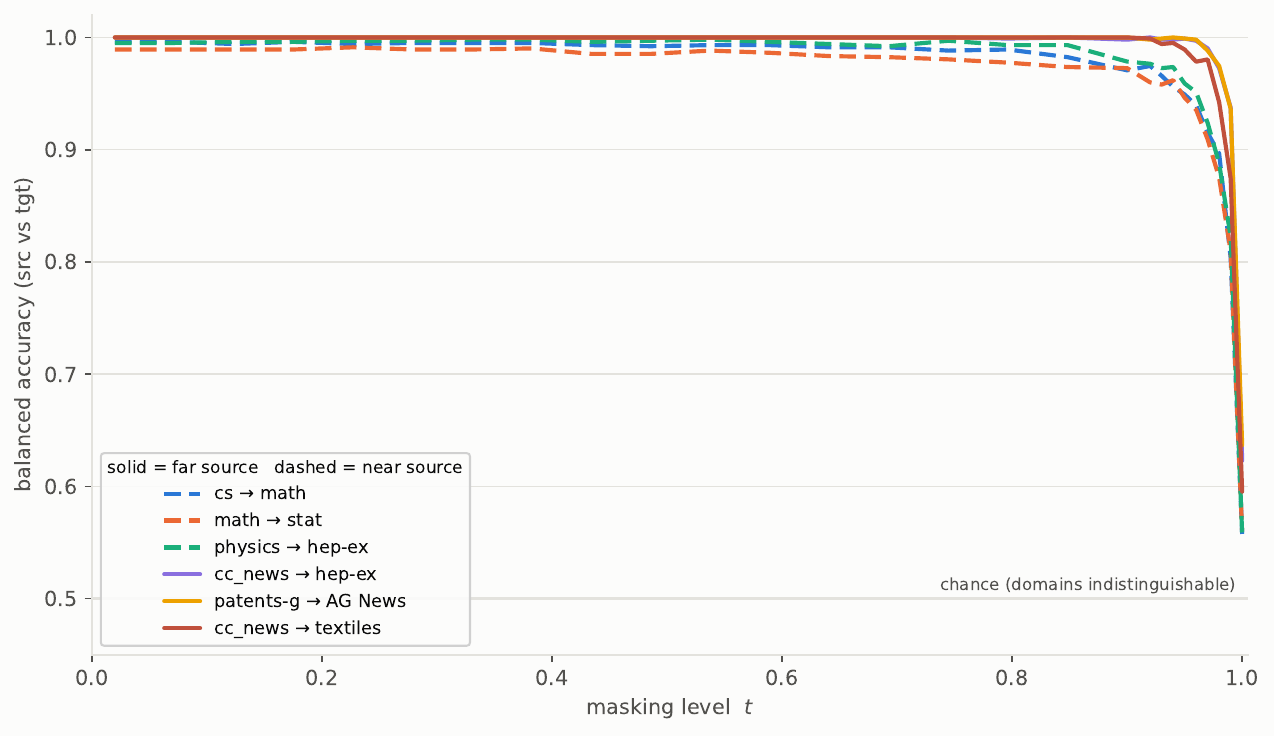}
    \caption{Masking kernel.}
    \label{fig:balacc-masking}
  \end{subfigure}
  \hfill
  \begin{subfigure}[t]{0.49\linewidth}
    \centering
    \includegraphics[width=\linewidth]{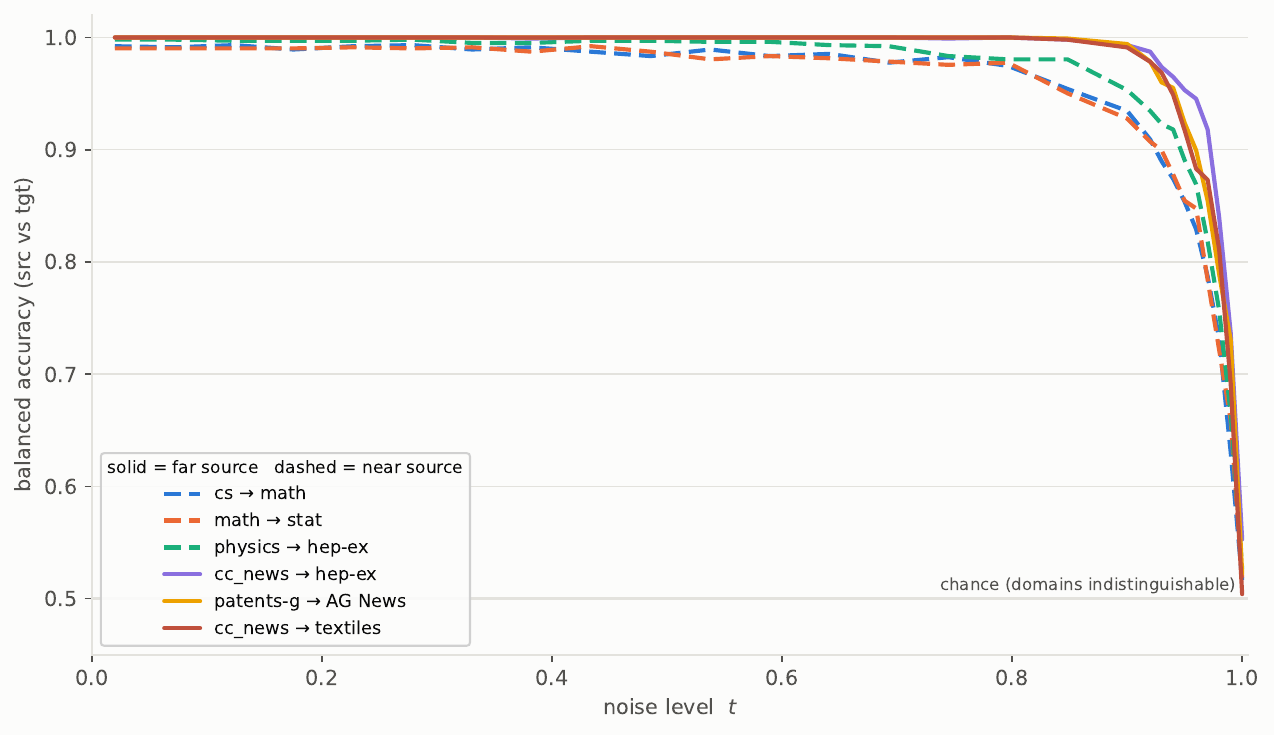}
    \caption{Uniform kernel.}
    \label{fig:balacc-uniform}
  \end{subfigure}
  \caption{\textbf{Domains stay separable almost to $t=1$.} Balanced accuracy of
  the domain classifier $c_t$ vs.\ noise level $t$, six source--target pairs of our text experiments; chance is $0.5$. Solid = far source, dashed = near. $c_t$ is a DDiT-style time-conditioned transformer ($256$ hidden, $4$ blocks, $4$ heads, attention-mean pooling) under the log-linear schedule, whose corrupted fraction is linear in $t$; $4000$ steps, batch $64$, lr $3\times10^{-4}$ with cosine decay and $10\%$ warmup, no label smoothing. The two panels differ only in the corruption kernel.}
  \label{fig:balacc}
\end{figure}

\subsection{Where the source helps, and what fine-tuning forgets}\label{app:forgettingandgt}
Table~\ref{tab:two-stage} shows that RefineMix beats fine-tuning on four pairs. Here we look inside the models to see why, noise level by noise level, and answer two questions: at which levels does the source help the target, and what does fine-tuning lose that RefineMix keeps?

For the first question, we compare data mixing with fine-tuning through the loss gap $g(t)$ of Section~\ref{sec:method-diagnostic}, measured on held-out target text (Figure~\ref{fig:lossgap-twostage}). Where $g(t) < 0$, adding the source improves the target at level $t$, and where $g(t) > 0$, the price of mixing has taken over. On the two arXiv pairs, $g$ turns positive almost exactly at the chosen cutoff, and on cc\_news~$\to$~Textiles it stays near or below zero, matching its high cutoff. On the two remaining pairs, the loss gap is only a rough guide: it turns positive too early on cc\_news~$\to$~Hep-Ex, and on Patents-G~$\to$~AG News it is positive everywhere, although RefineMix at $0.2$ still beats fine-tuning.

For the second question, we measure how much of the source knowledge each model keeps, by comparing its per-level cross-entropy on held-out source text with that of the shared source checkpoint (Figure~\ref{fig:forgetting}). Fine-tuning ends above the checkpoint on every pair and at every noise level: the small target corpus overwrites what pretraining learned, most at low and middle noise, where the denoiser predicts tokens from a rich context. RefineMix, in contrast, stays at or slightly below the checkpoint below its cutoff, and moves away only above it, where the source is no longer admitted. Hence the two models differ exactly where the schedule says they should. Since only final checkpoints were kept, these curves show how much is lost, not when.

\begin{figure}[H]
    \centering
    \includegraphics[width=0.9\linewidth]{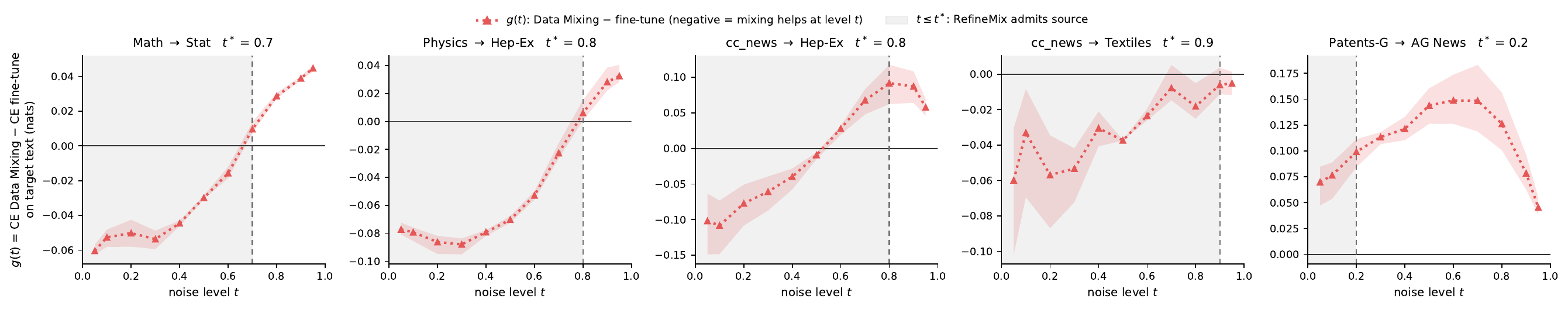}
    \caption{\textbf{Where the source helps the target.} For each pair of Table~\ref{tab:two-stage}, data mixing and fine-tuning start from the same source-pretrained checkpoint (masked kernel, mean over three seeds). We plot the loss gap $g(t)$, the per-level cross-entropy of data mixing minus that of fine-tuning, on held-out target text. Below zero, the source improves the target at that noise level; above zero, the price of mixing has taken over. On the arXiv pairs, the gap turns positive at the chosen cutoff; elsewhere, it is a rough guide.}
    \label{fig:lossgap-twostage}
\end{figure}

\begin{figure}[H]
    \centering
    \includegraphics[width=0.9\linewidth]{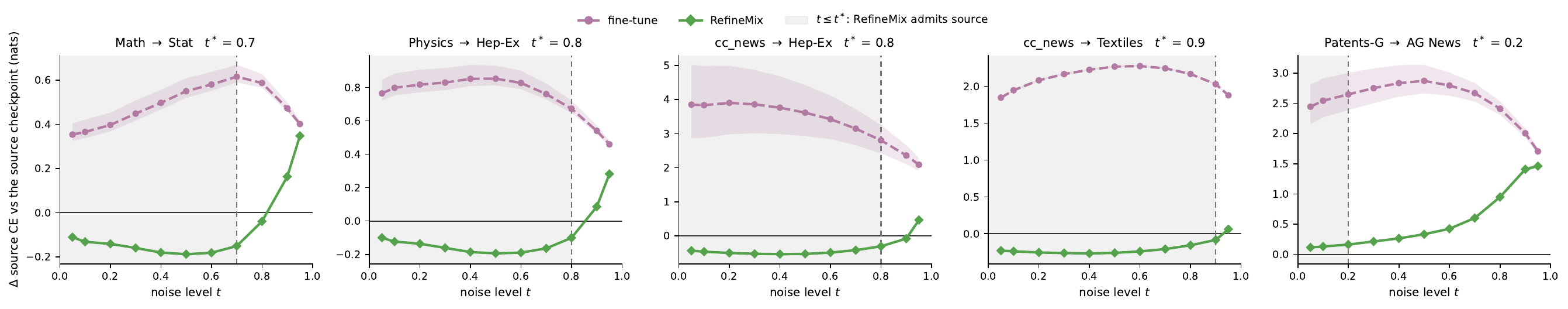}
    \caption{\textbf{RefineMix keeps what the source checkpoint knows.} For each pair of Table~\ref{tab:two-stage}, fine-tuning and RefineMix start from the same source-pretrained checkpoint (masked kernel, mean over three seeds). We plot the change in per-level cross-entropy on held-out source text between the checkpoint and the end of training; positive values mean that source knowledge was lost. The dashed line marks the cutoff of RefineMix, and the shaded region is where it admits the source. Fine-tuning loses source knowledge at every level, whereas RefineMix keeps it below the cutoff and moves away only above it.}
    \label{fig:forgetting}
\end{figure}

\section{Samples of the adapted pretrained models}\label{app:pretrained-samples}
Table~\ref{tab:paper} compares the models by their target perplexity. Since a model can fit target text and still write source text, we also sample from each adapted model and check what it writes. For this, we draw samples from pure noise over a sweep of sampling temperatures, and measure both whether they belong to the target domain (Figure~\ref{fig:pretrained-ontarget}) and how fluent and diverse they are (Figure~\ref{fig:pretrained-frontier}).

\begin{figure}[H]
    \centering
    \includegraphics[width=\linewidth]{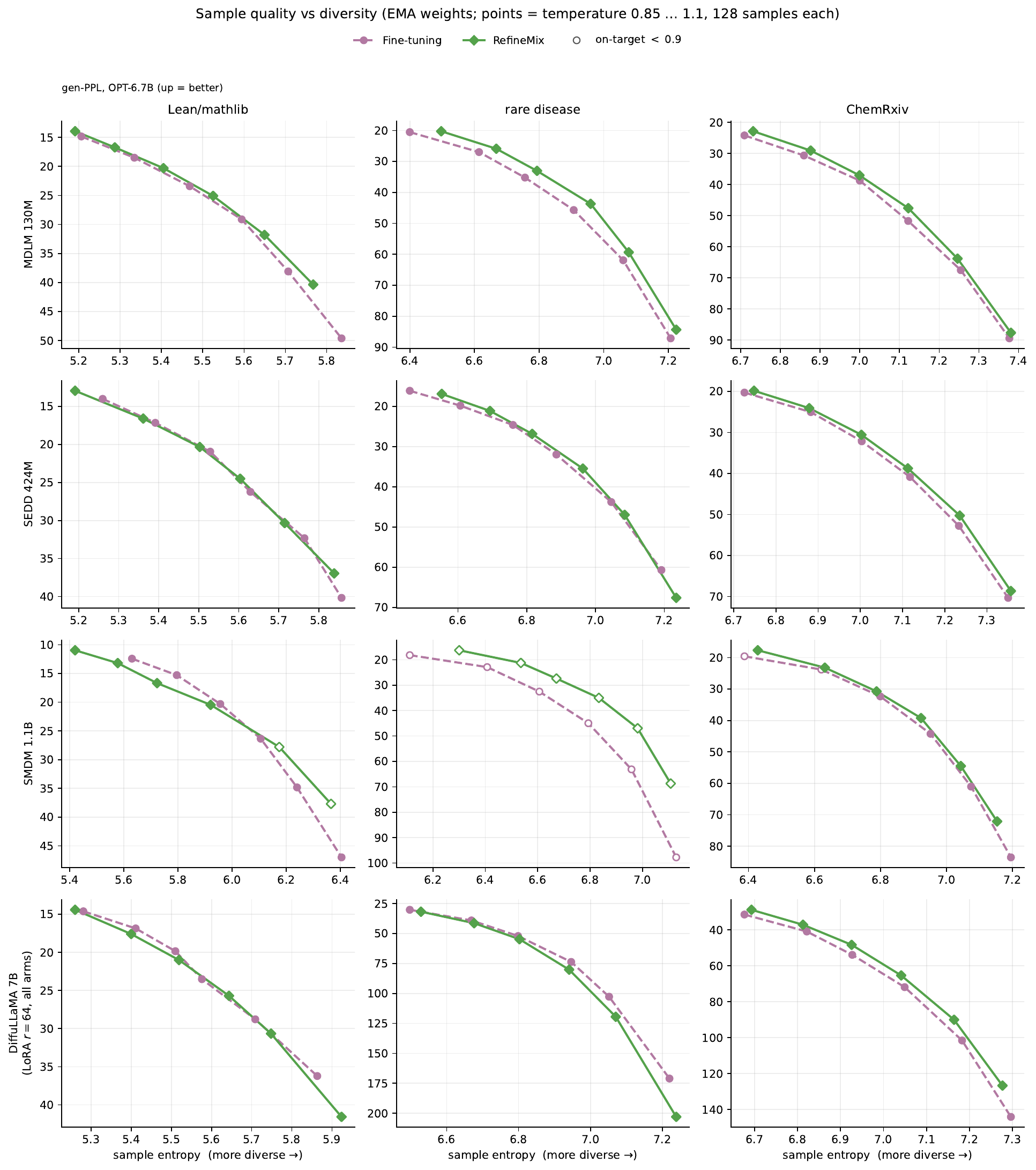}
    \caption{\textbf{Generation frontiers of the adapted models of Table~\ref{tab:paper}.} Each of the four pretrained models is adapted to each of the three targets, and we sample from pure noise over a sweep of sampling temperatures (mean over three seeds). We plot the generative perplexity of the samples, scored by OPT-2.7B, against their entropy, for fine-tuning and RefineMix. Perplexity is inverted, so up and to the right means more fluent and more diverse samples, and the dashed line marks real target text.}
    \label{fig:pretrained-frontier}
\end{figure}

\begin{figure}[H]
    \centering
    \includegraphics[width=\linewidth]{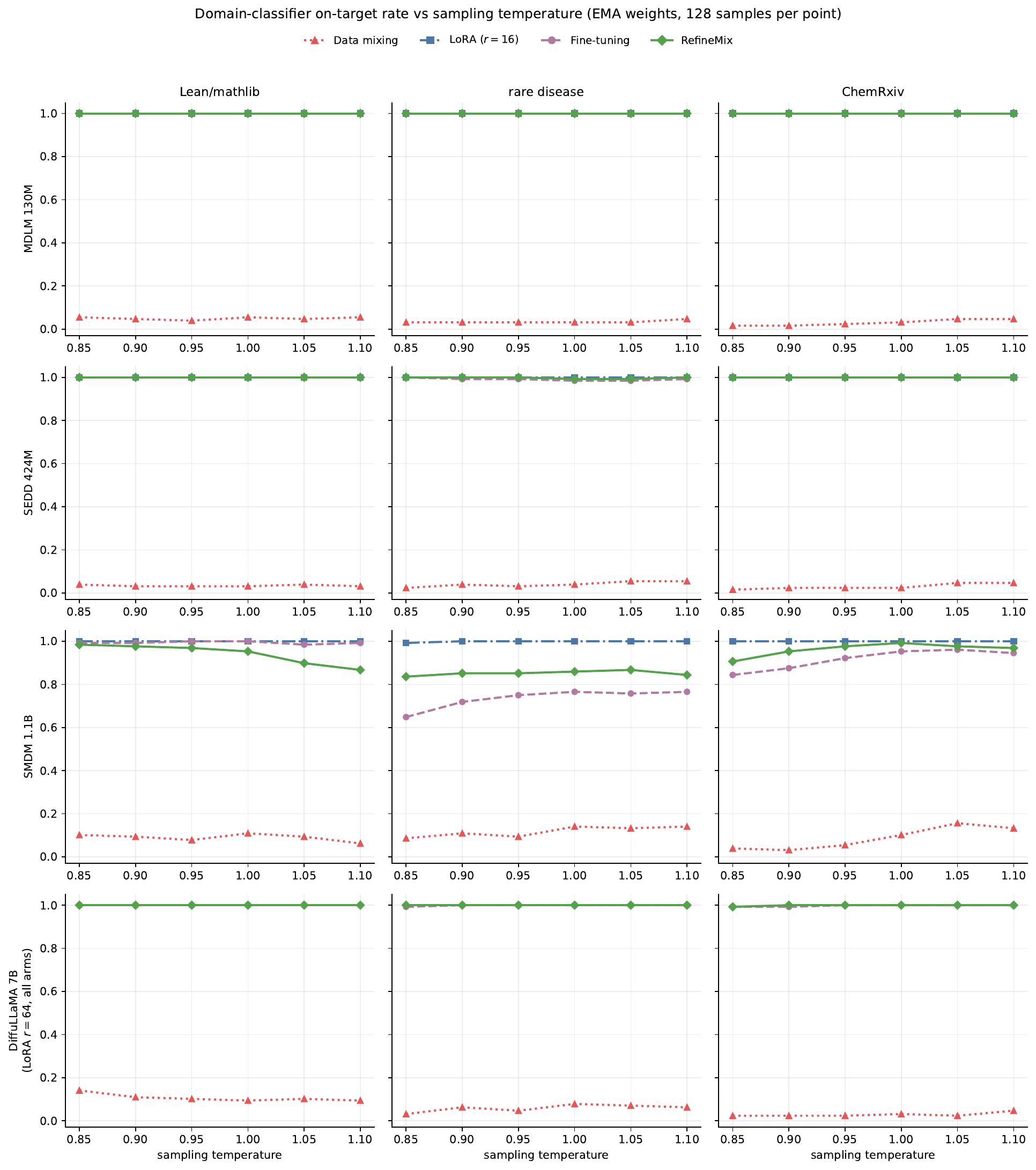}
    \caption{\textbf{The adapted models of Table~\ref{tab:paper} stay on target.} Each of the four pretrained models is adapted to each of the three targets, and we sample from pure noise over a sweep of sampling temperatures (mean over three seeds). We plot the on-target rate, the fraction of samples that a domain classifier assigns to the target rather than the source, against the sampling temperature, for all methods.}
    \label{fig:pretrained-ontarget}
\end{figure}

\clearpage
\section{Posteriors and losses for the two kernels}\label{app:kernels}
The marginals \eqref{equ_bg:forward} arise from a Markov chain that, between times $s < t$, keeps a token with probability $\alpha_{t\mid s} \coloneqq \alpha_t/\alpha_s$ and otherwise resamples it from $\boldsymbol\pi$. Let $\mathbf 1$ be the all-ones vector and $\langle\cdot,\cdot\rangle$ and $\odot$ the inner and Hadamard products. Both kernels share the posterior of the forward process, obtained from this chain by Bayes' rule,
\begin{align}
p(\mathbf z_s \mid \mathbf z_t, \mathbf z_0) = \mathrm{Cat}\!\left(\mathbf z_s;\;
\frac{\big[\alpha_{t\mid s}\mathbf z_t + (1-\alpha_{t\mid s})\langle \mathbf z_t,\boldsymbol\pi\rangle\mathbf 1\big]
\odot \big[\alpha_s \mathbf z_0 + (1-\alpha_s)\boldsymbol\pi\big]}
{\alpha_t \langle \mathbf z_t, \mathbf z_0\rangle + (1-\alpha_t)\langle \mathbf z_t, \boldsymbol\pi\rangle}\right),
\label{equ_bg:posterior}
\end{align}
and in each case the KL in \eqref{equ_bg:loss} simplifies.

\noindent\textbf{Masked (absorbing-state) diffusion.}
With $\boldsymbol\pi = \mathbf m$ the posterior is
\begin{align}
p(\mathbf z_s \mid \mathbf z_t, \mathbf z_0) =
\begin{cases}
\mathrm{Cat}\!\left(\mathbf z_s;\, \mathbf z_t\right), & \text{if } \mathbf z_t \neq \mathbf m,\\[6pt]
\mathrm{Cat}\!\left(\mathbf z_s;\, \dfrac{(1-\alpha_s)\,\mathbf m + (\alpha_s-\alpha_t)\,\mathbf z_0}{1-\alpha_t}\right), & \text{if } \mathbf z_t = \mathbf m,
\end{cases} \label{equ_bg:posterior-masked}
\end{align}
so an unmasked token is carried over unchanged and a masked one is
revealed with probability $(\alpha_s-\alpha_t)/(1-\alpha_t)$.
Constraining $\mathbf x_\theta$ to place no mass on \texttt{[MASK]} and
to copy unmasked tokens, the KL in \eqref{equ_bg:loss} becomes a
weighted cross-entropy: it vanishes at unmasked positions and equals
$\tfrac{\alpha_{s(i)}-\alpha_{t(i)}}{1-\alpha_{t(i)}}
\big(-\log\langle\mathbf x_\theta(\mathbf z_{t(i)}), \mathbf z_0\rangle\big)$
at masked ones~\citep{sahoo2024simpleeffectivemaskeddiffusion,shi2025simplifiedgeneralizedmaskeddiffusion}.
In continuous time this is
\begin{align}
\mathcal{L}_{\text{masked}}
= \mathbb{E}_{t \sim \mathcal U(0,1)}\,\mathbb{E}_{q}
\left[\frac{\alpha_t'}{1-\alpha_t}
\sum_{\ell:\, \mathbf z_t^{\ell} = \mathbf m}
\log\langle\mathbf x_\theta(\mathbf z_t)^{\ell}, \mathbf z_0^{\ell}\rangle\right],
\label{eq:app-loss-masked}
\end{align}
the objective used in all our experiments.

  \noindent\textbf{Uniform-state diffusion.}
With $\boldsymbol\pi = \mathbf 1/N$ no state is absorbing and the
posterior is~\citep{lou2024discrete,schiff2025simpleguidance,sahoo2025diffusionduality}
\begin{align}
p(\mathbf z_s \mid \mathbf z_t, \mathbf z_0) = \mathrm{Cat}\!\left(\mathbf z_s;\;
\frac{N\alpha_t\,\mathbf z_t \odot \mathbf z_0 + (\alpha_{t\mid s}-\alpha_t)\,\mathbf z_t + (\alpha_s-\alpha_t)\,\mathbf z_0 + (1-\alpha_{t\mid s})(1-\alpha_s)\,\frac{\mathbf 1}{N}}
{N\alpha_t\,\langle\mathbf z_t, \mathbf z_0\rangle + 1-\alpha_t}\right),
\label{equ_bg:posterior-uniform}
\end{align}
which places mass on the current token, on the clean token, and
uniformly on all others. Here the KL does not collapse to a
cross-entropy on a subset of positions, since every position carries
signal, and it is evaluated in full at all $L$ positions
\citep{schiff2025simpleguidance,sahoo2025diffusionduality}.

\section{Proofs}\label{app:proofs}
Both proofs concern the ideal model of Section~\ref{sec:method-setup}: every denoiser is an exact distribution over whole sequences, and $t_{\max}$ lies on the time grid $0 = t_0 < t_1 < \dots < t_K = 1$ of the sampler. As in Section~\ref{sec:background}, a reverse step with denoiser $D$ is
\begin{equation}\label{eq:app-step}
D(x_{t_{i-1}} \mid x_{t_i}) = \sum_{x_0} p(x_{t_{i-1}} \mid x_{t_i}, x_0)\, D(x_0 \mid x_{t_i}),
\end{equation}
which, since the forward process is Markov, is the exact reverse step of the forward chain started at $D$.

\subsection{Proof of Theorem~\ref{thm:safe}}\label{app:proof-safe}
\thmsafe*
\begin{proof}
\noindent\textbf{Classifier.} By Bayes' rule, with $m(x) = \lambda\, q(x) + (1-\lambda)\, p(x)$ at any level,
\begin{equation}\label{eq:app-c}
c(x) = \frac{\lambda\, q(x)}{m(x)}, \qquad 1 - c(x) = \frac{(1-\lambda)\, p(x)}{m(x)}.
\end{equation}

\noindent\textbf{Fact 1.} Since $q(x_t \mid x_0) = p(x_t \mid x_0)$,
\begin{align*}
m(x_0 \mid x_t)
&= \frac{p(x_t \mid x_0)\,\big[\lambda\, q(x_0) + (1-\lambda)\, p(x_0)\big]}{m(x_t)}
 = \frac{\lambda\, q(x_t)}{m(x_t)} \cdot \frac{p(x_t \mid x_0)\, q(x_0)}{q(x_t)} + \frac{(1-\lambda)\, p(x_t)}{m(x_t)} \cdot \frac{p(x_t \mid x_0)\, p(x_0)}{p(x_t)} \\
&= c(x_t)\, q(x_0 \mid x_t) + \big(1 - c(x_t)\big)\, p(x_0 \mid x_t).
\end{align*}

\noindent\textbf{Fact 2.} \emph{Output.} Since $p(x_0)\, q(x_0) = 0$, we have $p(A) = 0$ and $q(A) = 1$, so by Fact~1
\begin{equation}\label{eq:app-output}
\sum_{x_0 \in A} m(x_0 \mid x_t)
= c(x_t) \underbrace{\sum_{x_0 \in A} q(x_0 \mid x_t)}_{=\,1} + \big(1 - c(x_t)\big) \underbrace{\sum_{x_0 \in A} p(x_0 \mid x_t)}_{=\,0}
= c(x_t).
\end{equation}
Below the cutoff $t_k = t_{\max}$, the sampler composes the steps \eqref{eq:app-step} with $D = m$, the exact reverse steps of the chain started at $m$, so
\[
\sum_{x_{t_1}, \dots, x_{t_{k-1}}} \prod_{i=1}^{k} m(x_{t_{i-1}} \mid x_{t_i}) = m(x_0 \mid x_{t_k}).
\]
\emph{Input.} Since $\alpha_1 = 0$, $p(x_1 \mid x_0) = \prod_{\ell} \mathrm{Cat}(\mathbf z^\ell_1; \boldsymbol\pi)$ for every $x_0$, so the sampler starts at $x_1 \sim p(x_1)$. Above the cutoff, every step \eqref{eq:app-step} with $D = p$ preserves the target marginal,
\[
\sum_{x_{t_i}} p(x_{t_{i-1}} \mid x_{t_i})\, p(x_{t_i}) = p(x_{t_{i-1}}), \qquad i > k,
\]
so by induction $x_{t_{\max}} \sim p(x_{t_{\max}})$.

\noindent\textbf{Theorem.} Combining input, composition and output with \eqref{eq:app-c}, with $x$ at level $t_{\max}$,
\begin{align*}
\mathbb{P}(\hat x_0 \in A)
&= \sum_{x} p(x) \sum_{x_0 \in A} m(x_0 \mid x)
 = \sum_{x} p(x)\, c(x)
 = \sum_{x} p(x)\, \frac{\lambda\, q(x)}{m(x)}
 = \frac{\lambda}{1-\lambda} \sum_{x} q(x)\, \frac{(1-\lambda)\, p(x)}{m(x)} \\
&= \frac{\lambda}{1-\lambda} \sum_{x} q(x)\, \big(1 - c(x)\big)
 = \frac{\lambda}{1-\lambda}\, C(t_{\max})
 \;\le\; \frac{\lambda}{1-\lambda}\, C(t') \qquad \text{for every } t' \ge t_{\max},
\end{align*}
where the inequality is Lemma~\ref{lem:monotone-collision}. No step depends on $\boldsymbol\pi$, so the result holds for both kernels.
\end{proof}

\begin{lemma}\label{lem:monotone-collision}
The overlap $C(t)$ is nondecreasing in $t$.
\end{lemma}
\begin{proof}
By \eqref{eq:app-c}, $C(t) = \sum_x h\big(q(x), p(x)\big)$ with $x$ at level $t$ and
\[
h(a, b) = \frac{(1-\lambda)\, a\, b}{\lambda a + (1-\lambda) b} = (1-\lambda) \Big(\frac{\lambda}{b} + \frac{1-\lambda}{a}\Big)^{-1},
\]
which is $(1-\lambda)$ times a weighted harmonic mean, hence concave and positively homogeneous, hence superadditive. For $t \le t'$, both marginals pass through the same forward kernel $K(y \mid x) = p(x_{t'} = y \mid x_t = x)$, so
\[
C(t') = \sum_{y} h\Big(\sum_{x} K(y \mid x)\, q(x),\; \sum_{x} K(y \mid x)\, p(x)\Big)
\;\ge\; \sum_{y} \sum_{x} K(y \mid x)\, h\big(q(x), p(x)\big) = C(t).
\]
\end{proof}

\subsection{Proof of Theorem~\ref{thm:gain}}\label{app:gain-proofs}
As in Section~\ref{sec:method-tradeoff}, each noise level is learned separately, a trained denoiser has the split form \eqref{eq:split} with a learned specific factor and a learned shared factor, and h.o.\ collects terms of order $o(1/N_T)$. Sums over $x_0$ with content $v$ run over $\{x_0 : v(x_0) = v\}$.

\begin{proof}
\noindent\textbf{Latent content.} Since $x_0$ determines $v$, $r(x_0 \mid v) = 0$ for $v \neq v(x_0)$, so $p(x_0) = p(v)\, r(x_0 \mid v)$ with $v = v(x_0)$, and the same for $q$. If $p(v)\, q(v) > 0$, every $x_0$ with $r(x_0 \mid v) > 0$ has $p(x_0)\, q(x_0) > 0$, which the distinct domains exclude. Hence $p(v)\, q(v) = 0$.

\noindent\textbf{The denoiser splits.} With $v = v(x_0)$ and $r(x_t \mid v) = \sum_{x_0} p(x_t \mid x_0)\, r(x_0 \mid v)$, Bayes' rule gives
\[
p(x_0 \mid x_t)
= \frac{p(x_t \mid x_0)\, p(v)\, r(x_0 \mid v)}{p(x_t)}
= \underbrace{\frac{p(v)\, r(x_t \mid v)}{p(x_t)}}_{p(v \mid x_t)}\;
  \underbrace{\frac{p(x_t \mid x_0)\, r(x_0 \mid v)}{r(x_t \mid v)}}_{r(x_0 \mid x_t, v)},
\]
and the same for $q$ and $m$, with $q(v)$ and $m(v) = \lambda\, q(v) + (1-\lambda)\, p(v)$; the second factor does not depend on the domain. Since $r(x_1 \mid v) = \prod_\ell \mathrm{Cat}(\mathbf z^\ell_1; \boldsymbol\pi)$ for every $v$, and $p(x_t \mid x_0) = \mathbb{1}\{x_t = x_0\}$ at $t = 0$, we have $p(v \mid x_1) = p(v)$ and $p(v \mid x_0) = \delta_{v(x_0)}$. Summing Fact~1 over the $x_0$ with content $v$, and using $q(v \mid x_t) = 0$ for every target content,
\begin{equation}\label{eq:app-content-blend}
m(v \mid x_t) = c(x_t)\, q(v \mid x_t) + \big(1 - c(x_t)\big)\, p(v \mid x_t) = \big(1 - c(x_t)\big)\, p(v \mid x_t), \qquad p(v) > 0.
\end{equation}

\noindent\textbf{What each model learns.} For a denoiser $D(x_0 \mid x_t) = S(v \mid x_t)\, R(x_0 \mid x_t, v)$ with $v = v(x_0)$, the chain rule of the KL divergence gives
\begin{align}
D_{\mathrm{KL}}\big(p(x_0 \mid x_t) \,\|\, D(x_0 \mid x_t)\big)
&= \sum_{x_0} p(x_0 \mid x_t) \log \frac{p(v \mid x_t)\, r(x_0 \mid x_t, v)}{S(v \mid x_t)\, R(x_0 \mid x_t, v)} \nonumber\\
&= \underbrace{D_{\mathrm{KL}}\big(p(v \mid x_t) \,\|\, S(v \mid x_t)\big)}_{\text{specific}}
 + \underbrace{\mathbb{E}_{v \sim p(v \mid x_t)}\, D_{\mathrm{KL}}\big(r(x_0 \mid x_t, v) \,\|\, R(x_0 \mid x_t, v)\big)}_{\text{shared}}.
\label{eq:app-chain}
\end{align}

\noindent\textbf{Effective sample size.} Below the cutoff, each target sample has weight $w_i = (1-\lambda)/N_T$ and each source sample $w_i = \lambda/N_S$, so $\sum_i w_i = 1$ and
\[
N_{\mathrm{eff}} = \frac{\big(\sum_i w_i\big)^2}{\sum_i w_i^2} = \Big[\frac{(1-\lambda)^2}{N_T} + \frac{\lambda^2}{N_S}\Big]^{-1},
\qquad
\lambda = \rho:\;\; \frac{N_T}{(N_T + N_S)^2} + \frac{N_S}{(N_T + N_S)^2} = \frac{1}{N_T + N_S}.
\]
Since $N_S \ge N_T$, $1/N_{\mathrm{eff}} \le \big((1-\lambda)^2 + \lambda^2\big)/N_T \le 1/N_T$, so $o(1/N_{\mathrm{eff}}) = o(1/N_T)$.
For a weighted maximum-likelihood fit $\hat\phi = \arg\max_\phi \sum_i w_i \log r_\phi(x_0^i \mid x_t^i, v^i)$ in which every sample carries the same Fisher information $I$ (Assumption~\ref{ass:rate}), the sandwich formula gives
\[
\mathrm{Cov}(\hat\phi) \approx \Big(\sum_i w_i I\Big)^{-1} \Big(\sum_i w_i^2 I\Big) \Big(\sum_i w_i I\Big)^{-1} = \frac{I^{-1}}{N_{\mathrm{eff}}},
\qquad
\mathbb{E}\, D_{\mathrm{KL}}\big(r_\phi \,\|\, r_{\hat\phi}\big) \approx \tfrac12 \operatorname{tr}\big(I\, \mathrm{Cov}(\hat\phi)\big) = \frac{\dim \phi}{2 N_{\mathrm{eff}}}.
\]

\noindent\textbf{Price.} The price needs no assumption on transfer. Since $L_p$ is the NELBO, $L_p(\theta) - L_p(p)$ is the KL divergence between the forward path of the target and the reverse path of $\theta$. The reverse paths of $p$ and $p_{\text{RefineMix}}$ agree above $t_{\max}$, and below it both fill in $x_{t_1}, \dots, x_{t_{k-1}}$ with the same forward bridge $p(\cdot \mid x_0, x_{t_{\max}})$, so only $x_0$ remains. On target sequences, Fact~1 gives $m(x_0 \mid x) = \big(1 - c(x)\big)\, p(x_0 \mid x)$, hence, with $x \sim p(x_{t_{\max}})$,
\begin{align*}
L_p(p_{\text{RefineMix}}) - L_p(p)
&= \mathbb{E}\, D_{\mathrm{KL}}\big(p(x_0 \mid x) \,\|\, m(x_0 \mid x)\big)
   && \text{(only $x_0$ remains)} \\
&= \mathbb{E}\big[-\log\big(1 - c(x)\big)\big]
   && \text{(Fact~1: $m(x_0 \mid x) = (1 - c(x))\, p(x_0 \mid x)$)} \\
&= \mathbb{E}\, D_{\mathrm{KL}}\big(p(v \mid x) \,\|\, m(v \mid x)\big)
   && \text{(by \eqref{eq:app-content-blend})} \\
&= B(t_{\max})
   && \text{(definition of $B$)} \\
&= -\log(1-\lambda) - D_{\mathrm{KL}}\big(p(x_{t_{\max}}) \,\|\, m(x_{t_{\max}})\big)
   && \text{(by \eqref{eq:app-c})} \\
&\ge \mathbb{E}\big[c(x)\big]
   && \text{($-\log(1-c) \ge c$)} \\
&= \mathbb{P}(\hat x_0 \in A)
   && \text{(Theorem~\ref{thm:safe})}.
\end{align*}
Hence the price is governed by the same classifier as Theorem~\ref{thm:safe}, and it bounds its cost. By the data-processing inequality, $D_{\mathrm{KL}}(p(x_t) \,\|\, m(x_t))$ is nonincreasing in $t$, so $B$ is nondecreasing, from $B(0) = 0$, since $m(x_0) = (1-\lambda)\, p(x_0)$ on target sequences, to $B(1) = -\log(1-\lambda)$, since $p(x_1) = m(x_1)$. Only the Markov forward process shared by both domains is used, so the price holds for both kernels.

\noindent\textbf{Gain.} Let $\theta^*_\tau$ be the population model at cutoff $\tau$, so $\theta^*_0 = p$ and $\theta^*_{t_{\max}} = p_{\text{RefineMix}}$. By \eqref{eq:app-chain}, the error of a trained denoiser splits into the error of its specific factor and of its shared factor. The specific factor is learned from the same $N_T$ target samples in both models, so its error $s(t)$ is the same for $\tau = 0$ and $\tau = t_{\max}$; by Assumption~\ref{ass:rate}, the shared factor has error $e(t)/N + o(1/N)$, with $N = N_{\mathrm{eff}}$ for $t \le \tau$ and $N = N_T$ otherwise. Hence
\[
\mathbb{E}\big[L_p(\hat\theta_\tau)\big]
= L_p(\theta^*_\tau) + \int_0^1 s(t)\, dt + \int_0^{\tau} \frac{e(t)}{N_{\mathrm{eff}}}\, dt + \int_{\tau}^{1} \frac{e(t)}{N_T}\, dt + o\Big(\frac{1}{N_T}\Big),
\]
and subtracting $\tau = t_{\max}$ from $\tau = 0$, with $L_p(\theta^*_{t_{\max}}) - L_p(\theta^*_0) = B(t_{\max})$ from the price,
\[
\Delta(t_{\max}) = \Big(\frac{1}{N_T} - \frac{1}{N_{\mathrm{eff}}}\Big) \int_0^{t_{\max}} e(t)\, dt - B(t_{\max}) + o\Big(\frac{1}{N_T}\Big).
\]

\noindent\textbf{Best cutoff.} Write $\kappa = 1/N_T - 1/N_{\mathrm{eff}} = \lambda(2-\lambda)/N_T - \lambda^2/N_S$, which decreases in $N_T$, and $F_\kappa(t) = \kappa \int_0^t e - B(t)$, which is continuous on $[0,1]$, so its largest maximizer $t^\star_\kappa$ exists. Let $\kappa_1 > \kappa_2$ and suppose $t_2 = t^\star_{\kappa_2} > t_1 = t^\star_{\kappa_1}$. Since $e \ge 0$,
\[
F_{\kappa_1}(t_2) - F_{\kappa_1}(t_1)
= \kappa_1 \!\int_{t_1}^{t_2}\! e - \big(B(t_2) - B(t_1)\big)
\;\ge\; \kappa_2 \!\int_{t_1}^{t_2}\! e - \big(B(t_2) - B(t_1)\big)
= F_{\kappa_2}(t_2) - F_{\kappa_2}(t_1) \;\ge\; 0,
\]
so $t_2$ also maximizes $F_{\kappa_1}$, which contradicts $t_1$ being the largest maximizer. Hence $t^\star$ is nondecreasing in $\kappa$, and so nonincreasing in $N_T$.  The argument uses no property of $B$, so this holds for any price that does not depend on $N_T$.

\end{proof}

\section{Testing Theorem~\ref{thm:safe} directly}\label{app:bound}
Theorem~\ref{thm:safe} rests on one identity, Fact~2 of Section~\ref{sec:method-purity}: on every noised sequence $x$, the mixed denoiser writes a source sequence with probability exactly $c(x)$. In this section, we test this identity on a trained model. For this, it is useful to take the model that is most likely to drift, a fully mixed model on CS~$\to$~Math with $\lambda = 0.9$ and $1\%$ of the target, and to hand it noised sequences whose content we control. For each $n$, we reveal $n$ target-specific tokens, mask the rest, and draw $256$ such noised sequences $x^{(j)}$ per seed, $768$ over the three seeds. We then compare two quantities: how often the model completes $x^{(j)}$ as a source sequence, and the source probability $c(x^{(j)})$ that the classifier assigns to it. By Fact~2, the two should agree.

\noindent\textbf{From the theorem to the plotted curves.} Neither quantity can be computed exactly, so we estimate both with a domain classifier. We train this classifier balanced, with prior $\mu = \tfrac12$, since a $9{:}1$ imbalance would leave the target class poorly resolved, which is exactly where the test must be sharp. Hence we shift its output to the prior $\lambda$ of the mixture before it estimates $c$:
\begin{align}
\sum_{x_0 \in A} m(x_0 \mid x)
&= c(x)
   && \text{(Fact~2)} \nonumber\\
c(x)
&= \frac{\lambda\, q(x)}{m(x)}
   && \text{(by \eqref{eq:resp})} \nonumber\\
\hat c^{\mu}(x)
&= \frac{\mu\, q(x)}{\mu\, q(x) + (1-\mu)\, p(x)}
   && \text{(classifier output, prior $\mu$)} \nonumber\\
\frac{q(x)}{p(x)}
&= \frac{1-\mu}{\mu} \cdot \frac{\hat c^{\mu}(x)}{1-\hat c^{\mu}(x)}
   && \text{(solve for the likelihood ratio)} \nonumber\\
\frac{c(x)}{1-c(x)}
&= \frac{\lambda}{1-\lambda} \cdot \frac{q(x)}{p(x)}
 = \frac{\lambda}{1-\lambda} \cdot \frac{1-\mu}{\mu} \cdot \frac{\hat c^{\mu}(x)}{1-\hat c^{\mu}(x)}
   && \text{(odds of \eqref{eq:resp})} \nonumber\\
c(x)
&= \frac{\lambda(1-\mu)\, \hat c^{\mu}(x)}
        {\lambda(1-\mu)\, \hat c^{\mu}(x) + (1-\lambda)\, \mu\, \big(1-\hat c^{\mu}(x)\big)}
   && \text{(prior shift, \citealp{saerens2002adjusting})}
   \label{eq:app-prior-shift}\\
\widehat{\mathrm{input}}
&= \frac{1}{768} \sum_{j=1}^{768} c\big(x^{(j)}\big)
   && \text{(classifier on the noised input)} \nonumber\\
\widehat{\mathrm{output}}
&= \frac{1}{768} \sum_{j=1}^{768} \hat c^{\mu}\big(\hat x_0^{(j)}\big)
   \;\approx\; \frac{1}{768} \sum_{j=1}^{768} \sum_{x_0 \in A} m\big(x_0 \mid x^{(j)}\big)
   && \text{(classifier on the completion)} \nonumber
\end{align}
Here $\hat x_0^{(j)} \sim m(\cdot \mid x^{(j)})$ is the completion of $x^{(j)}$, and the classifier's score on the clean completion serves as a soft indicator of $\hat x_0^{(j)} \in A$. Averaged over noised target sequences instead of a fixed $n$, the input side is exactly the probability of Theorem~\ref{thm:safe}, $\mathbb{E}_{x \sim p(x_t)}[c(x)] = \frac{\lambda}{1-\lambda}\, C(t)$.

\noindent\textbf{Why the input side is an upper bound.} Theorem~\ref{thm:safe} is an equality for the optimal classifier $c$, but we can only estimate $c$ with a trained classifier $\hat c$, which sees less. If $\hat c$ is calibrated, it overestimates the source rate on noised target sequences, $\mathbb{E}_{x \sim p(x_t)}[c(x)] \le \mathbb{E}_{x \sim p(x_t)}[\hat c(x)]$. This follows from Jensen's inequality: writing $p(x) = (1-c(x))\, m(x)/(1-\lambda)$, both sides become expectations under $m$ of the concave function $u(1-u)$, at $c$ and at $\hat c = \mathbb{E}_m[c \mid \hat c]$. Hence we expect the input curve to lie at or above the output curve, and to touch it where the classifier is certain.

\noindent\textbf{Reading the plot.} Under masking, both curves start near $\lambda$ when nothing is revealed and separate as soon as target tokens appear. The generated curve falls fastest: four revealed target tokens bring it close to zero, and from eight on the model no longer writes source text. The upper bound falls more slowly, because the trained classifier still hedges where the sampler has already committed, and the gap is widest between two and four tokens. Under uniform noise, both curves stay near $\lambda$ for the first token, and the generated curve needs twelve to sixteen tokens to reach zero, because scrambled tokens hide which positions carry evidence, so the upper bound is also looser. The only point at which the generated curve lies above the upper bound is masking with nothing revealed. There the optimal classifier equals $\lambda$ exactly, so the gap reflects the error of the trained classifier, not a violation of the theorem. Hence the generated curve stays below the upper bound, as Theorem~\ref{thm:safe} predicts, and the gap between them is the price of a classifier that sees less than the optimal one.

\begin{figure}[H]
    \centering
    \includegraphics[width=0.9\linewidth]{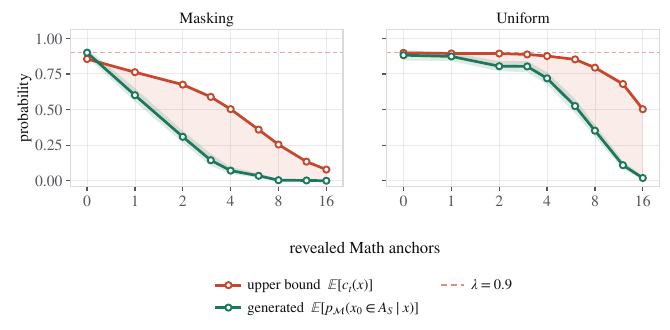}
    \caption{\textbf{Testing Theorem~\ref{thm:safe} directly} (CS $\to$ Math, three seeds). A fully mixed model ($\lambda = 0.9$, $1\%$ of the target) completes noised sequences in which $n$ target-specific tokens are revealed and the rest is masked. We plot two estimates of how often a completion is source text, against $n$: the classifier's source probability of the noised input, shifted to the prior $\lambda$ \eqref{eq:app-prior-shift}, and the classifier's score of the completion. By Theorem~\ref{thm:safe}, the two agree for the optimal classifier; a trained classifier makes the input side an upper bound. Bands are bootstrap intervals over $256$ samples per 3 seeds; the dashed line marks $\lambda = 0.9$.}
    \label{fig:purity-bound}
\end{figure}

\section{Separability under masking and Gaussian noise}\label{app:beyond}
Theorem~\ref{thm:safe} makes no use of discreteness, while the comparison in Figure~\ref{fig:main-figure}b varies the data and the noise process at the same time. To disentangle the two, we turn to images, where both noise processes can be applied to the same data. We use CIFAR-10 vs.\ SVHN as the main pair and CIFAR-10 vs.\ CIFAR-100 as a closer pair. We view these results as supporting evidence for the role of separability; they do not identify its cause.

\begin{figure}[H]
  \centering
  \begin{subfigure}[t]{0.3\linewidth}
    \centering
    \includegraphics[width=\linewidth]{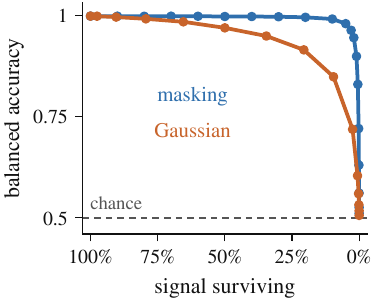}
    \caption{Masking keeps the domains separable longer.}
    \label{fig:beyond-separability}
  \end{subfigure}
  \hspace{0.06\linewidth}
  \begin{subfigure}[t]{0.3\linewidth}
    \centering
    \includegraphics[width=\linewidth]{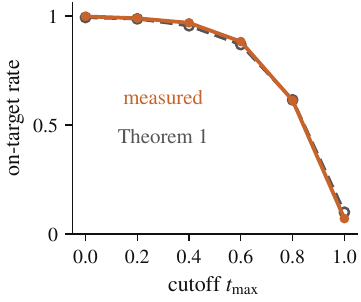}
    \caption{Theorem~\ref{thm:safe} predicts the on-target rate.}
    \label{fig:beyond-ontarget}
  \end{subfigure}
  \caption{\textbf{Separability beyond discrete diffusion} (CIFAR-10 vs.\ SVHN). \textbf{(a)} Balanced accuracy of a domain classifier on the same images under masking and under Gaussian noise, as a function of the fraction of surviving signal (three seeds; the min--max band is at most $0.009$ wide).  \textbf{(b)} Gaussian RefineMix with CIFAR-10 as target and SVHN as source ($\lambda = 0.9$): fraction of generated samples in the target domain as a function of the cutoff, measured and predicted by Theorem~\ref{thm:safe} (one seed, $2{,}000$ samples; the $95\%$ intervals are at most $\pm 0.022$).}
  \label{fig:beyond}
\end{figure}

\noindent\textbf{Separability under both noise processes.} We first measure how long a domain classifier can distinguish the two domains under each noise process. Masking replaces each of the $3{,}072$ colour values by \texttt{[MASK]} with probability $t$, and the Gaussian process is variance-preserving with $\bar\alpha_t = \cos^2(\pi t/2)$. Because $t$ has a different meaning for the two processes, we compare them at equal fractions of surviving signal, namely $1-t$ under masking and $\bar\alpha_t$ under Gaussian noise. Under masking, balanced accuracy remains above $0.9$ until only $1\%$ of the signal survives, whereas under Gaussian noise it already drops below $0.9$ at about $20\%$ (Figure~\ref{fig:beyond-separability}). Masking also preserves separability longer on the closer pair CIFAR-10 vs.\ CIFAR-100, although there the two classifiers already differ on clean images ($0.776$ vs.\ $0.810$). The size of the gap between the two noise processes thus depends on the domain pair.

\noindent\textbf{RefineMix with Gaussian diffusion.} Next, we check whether Theorem~\ref{thm:safe} also describes a trained Gaussian diffusion model. We train RefineMix at six cutoffs with CIFAR-10 as target and SVHN as source, sample with an ancestral sampler, and measure the fraction of samples that a clean classifier assigns to CIFAR-10. The theorem predicts this fraction to be $1 - \mathbb{E}_{x \sim p(x_{t_{\max}})}[c(x)]$, which we estimate using the noise-level classifier described above, shifted to the prior $\lambda$ as in \eqref{eq:app-prior-shift}. Measured and predicted values agree to within $0.03$ at every cutoff (Figure~\ref{fig:beyond-ontarget}). Since the target is the full CIFAR-10 training set, this experiment only probes where the source is safe to use; whether the source also improves the target is a separate question. The safe cutoff is lower than in our text experiments, as the image domains become inseparable earlier under Gaussian noise. This is in line with the observation that diffusion models commit to coarse structure early in the reverse process~\citep{biroli2024dynamical,sclocchi2025phase}.

\noindent\textbf{Details.} For each domain pair, we use $50{,}000$ training and $10{,}000$ test images per domain. The classifier is a noise-conditioned pre-activation ResNet-18 ($13.3$M parameters), trained identically for both noise processes for $40$k steps with $t \sim \mathcal U(0,1)$ and three seeds; only its input layer differs. The generator is a DDPM-style U-Net ($25.8$M parameters, v-prediction), trained for $100$k steps per cutoff and sampled with $256$ ancestral steps from the EMA weights. The clean classifier used to assign samples to a domain reaches $0.999$ balanced accuracy.

\end{document}